\documentclass{article} 
\usepackage{iclr2022_sty,times}

\usepackage{amsmath,amsfonts,bm}

\def\eqref#1{equation~\ref{#1}}

\def\1{\bm{1}}

\def\va{{\bm{a}}}
\def\vb{{\bm{b}}}
\def\vc{{\bm{c}}}

\def\vp{{\bm{p}}}

\def\vs{{\bm{s}}}

\def\vu{{\bm{u}}}

\def\vx{{\bm{x}}}

\DeclareMathAlphabet{\mathsfit}{\encodingdefault}{\sfdefault}{m}{sl}
\SetMathAlphabet{\mathsfit}{bold}{\encodingdefault}{\sfdefault}{bx}{n}

\DeclareMathOperator*{\argmax}{arg\,max}
\DeclareMathOperator*{\argmin}{arg\,min}

\usepackage{hyperref}
\hypersetup{colorlinks,linkcolor={blue},citecolor={blue},urlcolor={blue}} 

\usepackage{url}
\usepackage{bbm}
\usepackage{amsmath, nccmath}
\usepackage{amsthm}
\usepackage{booktabs}
\usepackage{multirow}
\usepackage{graphicx}
\usepackage{wrapfig}
\usepackage{enumitem}
\usepackage[labelfont=bf]{caption}

\usepackage[ruled, vlined, linesnumbered]{algorithm2e}
\SetKwComment{Comment}{$\triangleright$\ }{}
\SetKwProg{Init}{Initialize}{}{}

\newtheoremstyle{exampstyle}
  {\topsep} 
  {\topsep} 
  {} 
  {} 
  {\bfseries} 
  {.} 
  {.3em} 
  {} 

\theoremstyle{exampstyle} \newtheorem{theorem}{Theorem}[section]

\newtheorem*{theorem*}{Theorem}
\newtheorem{remark}{Remark}

\title{Low-Cost Algorithmic Recourse for Users With Uncertain Cost Functions}

\author{Prateek Yadav, \ Peter Hase, {\normalfont and} Mohit Bansal \\
Department of Computer Science\\
UNC Chapel Hill\\
\texttt{\{prateek,peter,mbansal\}@cs.unc.edu}}

\iclrfinalcopy 
\begin{document}

\maketitle

\begin{abstract}
People affected by machine learning model decisions may benefit greatly from access to recourses, i.e. suggestions about what features they could change to receive a more favorable decision from the model. 
Current approaches try to optimize for the cost incurred by users when adopting a recourse, but they assume that all users share the same cost function. This is an unrealistic assumption because users might have diverse preferences about their willingness to change certain features.
In this work, we introduce a new method for identifying recourse sets for users which does not assume that users' preferences are known in advance.
We propose an objective function, \textit{Expected Minimum Cost} (EMC), based on two key ideas: (1) when presenting a set of options to a user, there only needs to be one low-cost solution that the user could adopt;
(2) when we do not know the user's true cost function, we can approximately optimize for user satisfaction by first sampling plausible cost functions from a distribution, then finding a recourse set that achieves a good cost for these samples.
We optimize EMC with a novel discrete optimization algorithm, \textit{Cost-Optimized Local Search} (COLS), which is guaranteed to improve the recourse set quality over iterations. 
Experimental evaluation on popular real-world datasets with simulated users demonstrates that our method satisfies up to 25.89 percentage points more users compared to strong baseline methods, while human evaluation shows that our recourses are preferred more than twice as often as the strongest baseline recourses.
Finally, using standard fairness metrics we show that our method can provide more fair solutions across demographic groups than baselines.\footnote{We provide our source code at: \scriptsize{\url{https://github.com/prateeky2806/EMC-COLS-recourse}}.}
\vspace{-10pt}
\end{abstract}

\section{Introduction}
\label{sec:intro}

Over the past few years ML models have been increasingly deployed to make critical decisions related to loan approval \citep{siddiqi2012credit}, insurance \citep{wsj2019lifeinsurance}, allocation of public resources \citep{chouldechova2018case,shroff2017predictive} and hiring decisions \citep{ajunwa2016hiring}. Decisions from these models have real-life consequences for the individuals (users) involved. As a result, there is a growing emphasis on explaining these models' decisions \citep{ribeiro2018anchors, Lundberg2017AUA, poulin2006visual} and providing \emph{recourse} for unfavorable decisions \citep{EUdataregulations2018,Karimi2020ModelAgnosticCE}.
A \textit{recourse} is an actionable plan that allows someone to change the decision of a deployed model to a desired alternative \citep{Wachter2017CounterfactualEW}, which is often presented to users as a set of counterfactuals. 
Recourses can be highly valuable for users in situations where model decisions determine important life outcomes.
Recourses are desired to be \textit{actionable}, \textit{feasible}, and \textit{non-discriminatory}. 
\textit{Actionable} means that only features which can be changed by the user are requested to be changed. These changes should also be possible under the data distribution. 
For example, \textit{Education level} cannot be decreased from a \textit{Masters} to \textit{Bachelors} degree but can be increased from \textit{Masters} to \textit{PhD}. It is also not actionable to change your \textit{Race} \citep{dice_fat}.
A recourse is \textit{feasible} if it is reasonably easy for the user to adopt, i.e. it is actionable 
and has a low cost for the user. 
\textit{Non-Discriminatory} means that the recourse method should be equitable across population subgroups.
There are many fairness metrics now available \citep{Hinnefeld2018EvaluatingFM}, e.g. the ratio between the average cost of recourse for two subgroups in a population.

\begin{wrapfigure}{r}{0.5\textwidth}
    \centering
    \vspace{-5pt}
    \includegraphics{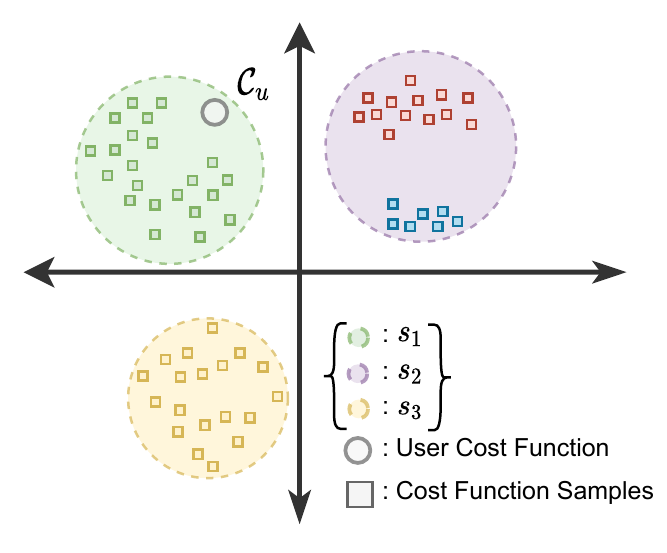}
    \caption{\label{fig:diagram} Method diagram showing the intuition behind the Expected Minimum Cost Objective. The squares denote cost function samples, which are the same color when they are similar.
    We aim to find a solution set of generated counterfactuals where each counterfactual does well under a particular region of cost function space (here, $\{s_1, s_2, s_3\}$). The shaded regions each represent a set of cost functions which a single $s_i$ caters to. In this case, we do not have enough counterfactuals to ``cover" every region of the cost function space, so a single counterfactual ($s_2$) must cater to two different regions. Here the user's hidden ground-truth cost function, $\mathcal{C}_\vu$, is served well by $s_1$. 
    }
\end{wrapfigure}

While we want recourses to be feasible for all users, it is difficult to directly optimize for a user's incurred cost unless we have access to their ground-truth cost function. In the absence of detailed cost function data,
prior work has used other heuristic objectives for feasibility. For instance, \citet{dice_fat} and \citet{Wachter2017CounterfactualEW} work with the underlying assumption that if the vector distance between the user's current state and the recourse is small, then recourse will be low cost. These works encourage this property via a \textit{proximity} objective. 
Meanwhile, \textit{sparsity} quantifies the number of features that require modification to implement a recourse \citep{dice_fat}. 
When providing multiple recourse options, \textit{feature diversity} in proposed recourses is used to counter uncertainty around the user cost function \citep{dice_fat, Cheng2021DECEDE}. The assumption is that if users are provided with solutions that change different subset of features, then they are more likely to find at least one feasible solution. 
Without access to ground-truth cost functions, a few recourse methods assume a known cost function that is shared by all users, then optimize for low-cost solutions under this function \citep{ar_fat, ares_neurips,karimi2020mintrecourse,karimi2020caterecourse,cui2015optimal}.
We believe it is crucial to have user-specific cost functions, because a global cost function is likely to poorly represent different users in a diverse population. 

In this work, we propose a method for identifying a user-specific recourse set that contains at least one good solution for the user. 
In the absence of data about users' true cost functions, we treat them as unknowns for the recourse method and assume they follow an underlying cost function distribution. We model this cost distribution via a highly flexible cost sampling procedure which makes minimal assumptions about user preferences. In addition, we provide users with an option to specify their preferred editable features or the complete cost function detailing the costs of transitions between features values.

We propose an objective function, \textit{Expected Minimum Cost} (EMC), allowing to approximately optimize for user incurred cost by first sampling plausible cost functions from a proposed distribution, then finding a recourse set that achieves a low cost over these cost function samples (in expectation). The EMC objective encourages the solution set to consist of counterfactuals that are each a good counterfactual under some particular cluster of cost functions from the sampling distribution.
Hence, if the user's ground-truth cost function is well represented by any of the clusters, then we will have \emph{some} counterfactual that is well suited (low-cost) to the user's cost function (shown in Figure \ref{fig:diagram}).
Lastly, we propose a discrete optimization method, \textit{Cost-Optimized Local Search} (COLS), in order to optimize for Expected Minimum Cost. COLS guarantees a monotonic reduction in the Expected Minimum Cost of the counterfactual set, leading to large empirical reductions in user-incurred cost.

To evaluate the effectiveness of our proposed techniques, we run experiments on two popular real-world datasets: Adult-Income \citep{adult} and COMPAS \citep{propublica_compas}. We compare our method with multiple strong baselines methods like Diverse Counterfactual Explanations (DICE) \citep{dice_fat}, Feasible and Actionable Counterfactual Explanations (FACE) \citep{face}, and Actionable Recourse (AR) \citep{ar_fat}. We evaluate these methods on existing metrics from the literature like diversity, proximity, sparsity, and validity along with 
two cost-based metrics (Section \ref{sec:setup}) and a human evaluation. In particular, we define the \textit{fraction of satisfied users} based on whether their cost of recourse is below a certain satisfiability threshold $k$. We also report \textit{coverage}, which is the fraction of users with at least one actionable recourse \citep{ares_neurips}.
Using simulated user true cost functions, we show that our method satisfies up to 25.89\% percentage points more users than strong baseline methods while covering up to 22.35\% more users across datasets.
Furthermore, our human evaluation shows that the recourses generated by our method are preferred by humans 57\% of the times as compared to 25\% for our strongest baseline (Actionable Recourse), a difference of 32 percentage points.

We also perform important ablations to show what fraction of the performance can be attributed to COLS optimization method or the EMC objective.
Additionally, we evaluate the robustness of our method to various distribution shifts that can occur between the population's true cost function distribution and our proposed cost sampling distributions (see Section \ref{sec:experiments}).
We find that our method is robust to these distribution shifts and generalizes well even when the cost functions for EMC objective are sampled from these shifted distributions.
Finally, we perform a fairness analysis of all the methods across demographic subgroups based on \textit{Gender} and \textit{Race}. Standard fairness metrics demonstrate that, in most comparisons, our method is more fair than the strongest baseline methods. \\
Our primary contributions in this paper are listed below.
\vspace{-5pt}
\begin{enumerate}[noitemsep,nolistsep]
    \item We evaluate user-incurred cost and fraction of satisfied users by means of user-specific cost functions, rather than a known global cost function.
    \item We propose a new objective function, Expected Minimum Cost (EMC), which approximately optimizes for user-incurred cost when their true cost function is not known, by sampling from a plausible distribution.
    \item We propose a discrete optimization method, \textit{Cost-Optimized Local Search} (COLS), which achieves up to 25.89\% percentage points higher user satisfaction relative to the next-best baseline. Under human evaluation, we find that our recourses are preferred more than twice as often as the strongest baseline recourses.
    \item We show that in most settings, our method provides more fair solutions across demographics subgroup than comparable recourse methods.
    
\end{enumerate}
\vspace{-7pt}
\section{Related Work} 
\label{sec:related}
Many methods now exist for generating recourses. We point readers to \citet{venkatasubramanian2020philosophical} for a philosophical basis for algorithmic recourse and to \citet{Karimi2021survey} for a comprehensive survey of the existing recourse methods.
Here, we distinguish our approach based on our recourse objectives, optimizer, and evaluation. 

\textbf{Objectives:}
The most prominent family of objectives for recourse includes distance-based objectives  \citep{Wachter2017CounterfactualEW,Karimi2020ModelAgnosticCE,Dhurandhar2018ExplanationsBO,dice_fat,rasouli2021care}. These methods primarily seek recourses that are close to the original data point. In DICE, \citet{dice_fat} provide users with a set of counterfactuals while trading off between proximity and feature diversity, which are distance-based objectives. Feature diversity based methods assume that changing different feature subsets in different recourses will increase the chance a user is satisfied by one of the options. \textit{In contrast, our method promotes diversity in the recourse set such that it has low cost for a wide variety of cost function samples from the proposed distribution.}

A second category of methods uses other heuristics based on the data distribution \citep{aguilar2020cold,gomez2020vice} to come up with counterfactuals. FACE constructs a graph from the given data and then tries to find a high-density path between points in order to generate counterfactuals \citep{face}. 
Lastly, the works closest to ours are the cost-based objectives, which capture feasibility in terms of the cost of recourse:
(1) \citet{cui2015optimal} define a cost function 
specifically for tree-based classifiers, which compares the different paths that two data points follow in a tree to obtain a classifier-dependent measure of cost.
(2) \citet{karimi2020mintrecourse,karimi2020caterecourse} take a causal intervention perspective on the task and define cost in terms of the normalized distance between the user state and the counterfactual. 
(3) \citet{ar_fat} define cost in terms of the number of changed features and frame recourse generation as an Integer Linear Program. (4) 
\citet{ares_neurips} infer global cost function from pairwise comparisons of features that are drawn from simulated users. 
However, they take a different approach to the recourse generation problem, which is to find a list of rules that can apply to any user to obtain a recourse, rather than specially generating recourses for each user as in this work. 
Importantly, all of these works assume there is a known global cost function that is shared by all users.
\textit{In our work, we drop this assumption, and each user has their own personal cost function.}

\textbf{Optimization:} Early work on recourse methods uses gradient-based optimization to search for counterfactuals close to a user's data point \citep{Wachter2017CounterfactualEW}. Several methods since then also use gradient-based optimization \citep{dice_fat,chen2020strategic}. Some recent approaches use tree-based techniques \citep{ares_neurips,von2020fairness,hashemi2020permuteattack,kanamoridace} or kernel-based methods \citep{dandl2020multi,gomez2020vice,ramon2019counterfactual}, while others employ some heuristic \citep{face,aguilar2020cold} to generate counterfactuals. A few works use autoencoders to generate recourses \citep{vae_1,Joshi2019TowardsRI}, while \cite{Karimi2020ModelAgnosticCE} and \cite{ar_fat} utilize SAT and ILP solvers, respectively. Drawing on general search principles \cite{pirlot_search}, \textit{we introduce a discrete optimization method (COLS) specialized but not limited to our EMC objective.}

\textbf{Evaluation:} Besides ensuring that recourses are classified as the desired outcome by a model (validity), the most prominent approaches to evaluate recourses rely on Distance-based metrics. In DICE, \citet{dice_fat} evaluate recourses according to their proximity, sparsity, and feature diversity. 
Meanwhile, several works directly consider the cost of the recourses, using a known global cost function as a metric, meaning that all users share the same cost function. \emph{This single cost function is used for both recourse generation and evaluation, i.e. the solutions are optimized and tested on the same cost function} \citep{cui2015optimal,karimi2020mintrecourse,karimi2020caterecourse}. 
Rather than assuming a cost function, \citet{ares_neurips} estimate a cost function from simulated pairwise feature comparisons, but this single estimate is used (for all users) for both optimizing recourses and evaluation. 
In contrast, \textit{we evaluate a recourse method by simulating user-specific cost functions which can vary greatly across users, and our method does not assume that these cost functions are known in advance.} 

\vspace{-7pt}
\section{Problem Statement}
\label{sec:problem_statement}

\textbf{Notation.}
\label{sec:notation}
We assume that we have a dataset with features $\mathcal{F} = \{f_1, f_2, ... f_k\}$.
Features can either be mutable, conditionally mutable, or immutable, according to the real-world causal processes that generate the data. 
For example, \textit{Race} is an immutable feature \citep{dice_fat}, \textit{Age} and \textit{Education} are conditionally mutable (cannot be decreased under any circumstances), and \textit{number of work hours} is mutable (can both increase and decrease). Note that continuous features are always discretized to integers.

\textbf{User Definition.}
\label{sec:user_definition}
A user is defined as a tuple $\vu = (\vs_{u}, \mathcal{C}^{*}_{u})$, where $\vs_u$ is the state vector of length $|\mathcal{F}|$ containing the user's features values and $\mathcal{C}^{*}_{u} = \{\mathcal{C}^{*(f)}_{u} ~|~ \forall f \in \mathcal{F}\}$ is their hidden ground-truth cost function (see Appendix Table \ref{tab:examples} for a qualitative example). 
Following past work \citep{ares_neurips}, we note that it may be difficult for users to precisely quantify their own cost functions in practice, but we do assume that it exists \cite{utilityMET}.

\textbf{Cost Function.}
\label{sec:cost_function}
In this work, we assume that each user has an inherent preference regarding the ease of changing a particular feature, where different users can have different preferences. Such differential preferences can be expressed via user-specific costs of transitioning between feature states. 
We define a \emph{cost function} for each user $\mathcal{C}: \mathbb{R}^{|\mathcal{F}|} \rightarrow \mathbb{R}$ as a set of feature-specific functions which provide the user-incurred cost when transitioning between feature states. Formally, a cost function is parametrized as a set of \emph{functions}, $\mathcal{C} = \{ \mathcal{C}^{(f)}(i, j): \mathbb{R}^{|f| \times |f|} \rightarrow [0,1] \cup \{\infty\} ~ | ~ \forall f \in \mathcal{F} \}$, where $\mathcal{C}^{(f)}(i, j)$ is the cost of transition from state $i \rightarrow j$  for the feature $f$ and $|f|$ is the number of values the feature can take. Here, $0$ means the transition has no associated cost incurred, $1$ means the transition is maximally difficult to make, and $\infty$ means the transition is infeasible.

\textbf{Transition Costs.} 
\label{sec:transition_cost}
Given two state vectors $\vs_i$, $\vs_j$ and any cost function $\mathcal{C}$, transition cost is the summation of transition costs for individual features, defined as $\text{Cost}(\vs_i, \vs_j ; \mathcal{C}) = \sum_{f \in \mathcal{F}} \mathcal{C}^{(f)}(s^{(f)}_i, s^{(f)}_j)$, where $\vs^{(f)}$ is the value of feature $f$ in the state vector. Now, given $\vs_u$, a recourse set $\mathcal{S}$, and a cost function $\mathcal{C}$, we suppose that the cost a user will incur is the \textit{minimum} transition cost across possible recourses, since a rational user will select the least costly option.
\begin{equation} 
    \vspace{-4pt}
    \text{MinCost}(\vs_{u}, \mathcal{S} ; \mathcal{C}) = \min\limits_{s_j \in \mathcal{S}} \text{Cost}(\vs_{u}, \vs_j ; \mathcal{C}).
    \label{eqn:min_cost}
    \vspace{-4pt}
\end{equation}

\textbf{Problem Definition.}
\label{sec:problem_definition}
For a given user $\vu$, our goal is to find a recourse set $\mathcal{S}_u$ such that,
\begin{align}
    \mathcal{S}_u = \argmin\limits_{\mathcal{S}}~ & \text{MinCost}(\vs_{u}, \mathcal{S} ; \mathcal{C}^*_u) \nonumber\\
    & \text{s.t.} ~~ \exists ~ s_i \in \mathcal{S} ~~ \text{s.t.} ~~ F(s_i) = 1 
\label{eqn:main}
\end{align} 
\vspace{-15pt}

where $F$ is the black-box ML model and $1$ is the desired outcome class. We want to offer users at least one counterfactual that is a good solution under their true cost function $\mathcal{C}^*_u$. 
Since do not know the user's true cost function, we minimize the expected cost for each user using a distribution over plausible cost functions. 

\section{Proposed Method: EMC and COLS}
\label{sec:our_method}
\subsection{Approximately Optimizing for User Cost Under Uncertain Cost Functions (EMC)}
\label{sec:approx_objective}
In most use cases, the true cost functions associated with each user \textbf{$\mathcal{C}^{*}_{u} \sim \mathcal{D}^*$ are unknown} to us and follow the \textbf{population's true cost function distribution $\mathcal{D}^*$ which is also not known}. Hence, we cannot exactly minimize the true user cost in Equation \ref{eqn:main}.
Instead, we propose to use a flexible cost function sampling distribution $\mathcal{D}_{train}$ which can model a large variety of cost functions. Given samples from this distribution, we can minimize the \textit{Expected Minimum Cost} (EMC) of a transition for a user. Formally for any user $\vu$ with state vector $\vs_u$, we optimize for
\begin{equation} 
    \mathbb{E}_{\mathcal{C}_i \sim \mathcal{D}_{train}}[\text{MinCost}(\vs_u, \mathcal{S} ; \mathcal{C}_i)]
    \label{eqn:obj}
\end{equation}
In practice, we employ Monte Carlo Estimation \citep{monte_carlo} to approximate this expectation by sampling $M$ cost functions $ \{\mathcal{C}_i\}_{i=1}^M$ from $\mathcal{D}_{train}$. Now, using Eq. \ref{eqn:min_cost}, we give our EMC objective as:
\begin{equation} 
    \text{EMC}(\vs_u, \mathcal{S} ; \{\mathcal{C}_i\}_{i=1}^M) 
    = \frac{1}{M}\sum_{i=1}^{M} \min\limits_{s_j \in \mathcal{S}} \text{Cost}(\vs_{u}, \vs_j ; \mathcal{C}) \label{eqn:mcmc}
\end{equation}

The EMC objective encourages the solution set to consist of counterfactuals that are each a good counterfactual under some particular cluster of cost functions which represent similar preferences in the cost function distribution. That is, the $\min$ term inside the summation allows for different recourses to cater to different regions of the cost function distribution. Hence, if the user's ground-truth cost function is well represented by any of the clusters, then we will have \emph{some} counterfactual that is well suited to the user's true cost function (see in Figure \ref{fig:diagram}). 

\begin{algorithm}[t!]
\footnotesize
\caption{Procedure for Sampling Cost Functions.\label{alg:sampling}}
\DontPrintSemicolon
\KwIn{State vector $\vs$, \textbf{Optional:} Preferred featured $\mathcal{F}_p$, feature preference scores $\vp$, cost distribution mixing weight $\alpha$}
\KwOut{Preference scores $\vp$ and the cost functions $\mathcal{C}$.}
\SetKwBlock{Begin}{function}{end function}
\Begin($\text{sampleCost}{(}\vs, \alpha=\text{None},  \mathcal{F}_p=\{\}, p=\text{None}{)}$)
{
    \uIf{$\mathcal{F}_p$ is \{\}}{
        $\mathcal{F}_p \sim RandomSubset(\mathcal{F}_{mutable})$ 
    }\;
    \uIf{$p$ is \text{None}}{
        $concentration = [1 ~if ~f \in \mathcal{F}_p ~else ~0 ~for ~f ~in ~\mathcal{F} ]$\;
        $p \sim Dirichlet(concentration)$
    }\;
    \uIf{$\alpha ~is~ \text{None}$}{
        $\alpha \sim Uniform(0,1)$
        
    }\;
    $\mathcal{C} = \{\} $\;
    
    \ForAll{$f_i \in \mathcal{F}$}
    {
        \Comment{\scriptsize Get means and variance for costs.}
        $\mu^{(f_i, Lin)}, \sigma^{(f_i, Lin)} \leftarrow LinCost(\vs, p^{(f_i)}, f_i, \mathcal{F}_p)$
        
        $\mu^{(f_i, Perc)}, \sigma^{(f_i, Perc)} \leftarrow PerCost(\vs, p^{(f_i)}, f_i, \mathcal{F}_p)$\\
        
        $\mu^{(f_i)} \longleftarrow \alpha * \mu^{(f_i, Lin)} + (1-\alpha) * \mu^{(f_i, Perc)}$
        
        $\sigma^{(f_i)} \longleftarrow \alpha * \sigma^{(f_i, Lin)} + (1-\alpha) * \sigma^{(f_i, Perc)}$
        
        \Comment{\scriptsize Beta parametrized with mean and variance}
        $\mathcal{C}^{(f_i)} \longleftarrow Beta(\mu^{(f_i)}, \sigma^{(f_i)})$
        
         $\mathcal{C} \gets \mathcal{C} \cup \mathcal{C}^{(f_i)}$
    }
    \Return{$p$, $\mathcal{C}_{p}$}\; 
}
\end{algorithm}

\subsection{Hierarchical Cost Sampling Procedure}
\label{sec:sampling}
To optimize for EMC, we need a plausible distribution $\mathcal{D}_{train}$ which is capable of generating cost functions which model a diverse set of cost functions. 
Past works like \citet{ar_fat} characterize transition cost for users in terms of \textit{percentile shift}. This cost is proportional to the change in a feature's percentile value as it is changed from an old value to a new one (see Algorithm \ref{alg:percentile}). 
In contrast, people also think of cost in terms of the number of steps between the initial and the final feature state. For instance, to change your education from \textit{High-School} to \textit{Masters}, there are two steps, \textit{High-School} to \textit{Bachelors} to \textit{Masters}. As the number of steps increases, the associated cost also increases for users. We assume the increase is linear in the number of steps, so we call this the \textit{Linear Cost} (see Algorithm \ref{alg:linear}). Next, we propose a sampling framework where we condition these costs based on preference scores for feature $f_i$, denoted by $p^{(f_i)}$, where preference scores represent individual users' willingness to change certain features and sum to 1 across features ($\sum_{f_i} p^{(f_i)} = 1$). Specifically, we scale these costs of transition (percentile and linear cost) between feature states by $(1-p^{(f_i)})$ which decreases the cost of transition for preferred features.
Building on these ideas, we propose three cost sampling distributions, $\mathcal{D}_{perc}$, $\mathcal{D}_{lin}$, $\mathcal{D}_{mix}$ which are highly flexible and can model very diverse types of cost functions and preference scores. The distributions $\mathcal{D}_{perc}$, $\mathcal{D}_{lin}$ are based on percentile and linear shifts, whereas $\mathcal{D}_{mix}$ is our most general framework which combines both linear and percentile shift costs with each user's mixing weight $\alpha$ (presented in Algorithm \ref{alg:sampling}).

Our $\mathcal{D}_{mix}$ sampling algorithm (Alg. \ref{alg:sampling}) thus proceeds as follows: we first generate a set of editable features for this user by taking a random subset of mutable features, then we sample preference scores $\vp$ for these editable features via a Dirichlet distribution, which represent the relative ease of editing each feature. Then, we model the final cost of transition of a feature from state $j \rightarrow k$ as a Beta distributed random variable with small noise (std = 0.01) which outputs a cost between 0 and 1. We use a small standard deviation at this step to inject extra noise into user preferences following the Dirichlet sampling.
The mean for this Beta distribution is given by the $\alpha$-weighted combination of the linear and percentile cost, where $\alpha$ is drawn from $Unif(0,1)$. 

We emphasize a few core properties of the $\mathcal{D}_{mix}$ distribution: (1) The distribution is very flexible. It is able to capture all possible feature subsets which different users might consider editable as well as possible mixing weights for combining linear and percentile costs, which is a much broader space of cost functions than considered in past work. Critically, the difficulty of editing each feature can range from ``trivial" to ``maximally difficult" (relative to other features), meaning almost all plausible user cost functions should be represented in the distribution.
(2) Our \textit{Linear} and \textit{Percentile} cost sampling functions yield monotonic costs, i.e. if the user has to make more drastic changes to the feature, then the associated cost will be higher.
(3) The user has an option to adapt this distribution to their needs by providing either the editable features or the feature scores $\vp$. 
Together, these properties allow us to represent a large space of plausible user cost functions, which helps us as it increases the probability of sampled cost functions being close to the user's true cost function.

\subsection{Search Methods for Finding Low Cost Counterfactuals.}
\label{sec:search_method}
\textbf{COLS:} To optimize for Expected Minimum Cost (Equation \ref{eqn:mcmc}), we propose two simple, efficient, and optimized discrete search algorithms \citep{pirlot_search}, namely \textit{Cost-Optimized Local Search} (COLS) and \textit{Parallel Cost-Optimized Local Search} (P-COLS) (refer to Algorithm \ref{alg:cols}).
COLS maintains a best set which will be the final recourse provided to the user. At each iteration, a candidate set is generated by locally perturbing each element of the best set with a Hamming distance of two, i.e. making small changes to two features at once. Then it is evaluated against the EMC objective. Instead of making a direct comparison between the best-set-so-far and the candidate set, at this point we evaluate whether any counterfactuals from the candidate set would improve the best set if we swapped out individual counterfactuals. 
Specifically, if the benefit of replacing $s_i \in \mathcal{S}_t$ with $s_j \in \mathcal{S}^{best}$ is positive then we make the replacement (see Algorithm \ref{alg:theorem}). The ability to assess the benefit of each candidate counterfactual is critical because it allows us to constantly update the best set using individual counterfactuals from a candidate set instead of waiting for an entire candidate set with lower EMC. 
For objectives like feature diversity, evaluating the benefit of individual replacement becomes expensive (see Appendix \ref{sec:supp_method}).
With COLS, we can guarantee that the EMC of the best set will monotonically decrease over time, which we formally state below:

\begin{theorem}[Monotonicity of Cost-Optimized Local Search Algorithm]
\label{thm:monotonicity}
Given the best set, $\mathcal{S}^{best}_{t-1} \in \mathbb{R}^{N \times d}$, the candidate set at iteration $t$, $\mathcal{S}_t \in \mathbb{R}^{N \times d}$, the matrix $\textbf{C}^b \in \mathbb{R}^{N \times M}$ and $\textbf{C} \in \mathbb{R}^{N \times M}$ containing the incurred cost of each counterfactual in $\mathcal{S}^{best}_{t-1}$ and $\mathcal{S}_{t}$ with respect to all the $M$ sampled cost functions $\{\mathcal{C}_i\}_{i=1}^{M}$, there always exist a  $\mathcal{S}^{best}_{t}$ constructed from $\mathcal{S}^{best}_{t-1}$ and $\mathcal{S}_t$ such that 
\begin{equation}
    \footnotesize 
    \text{EMC}(\vs_u, \mathcal{S}^{best}_{t} ; \{\mathcal{C}_i\}_{i=1}^M) \leq \text{EMC}(\vs_u, \mathcal{S}^{best}_{t-1} ; \{\mathcal{C}_i\}_{i=1}^M)   \nonumber
\end{equation}
\label{theorem}
\vspace{-15pt}
\end{theorem}
For the proof of the theorem, please refer to Appendix \ref{proof}.

\textbf{P-COLS:} The P-COLS method is a variant of COLS which starts multiple parallel runs of COLS with different initial sets. With a given computational budget, each run is allocated a fraction of the budget. The run with the least objective value is selected to provide recourses for the user.

\section{Experiments}

\label{sec:experiments}
\subsection{Experimental Setup}
\label{sec:setup}

\textbf{Dataset:}
\label{setup:dataset}
We conduct our experiments on the Adult-Income \citep{adult} and COMPAS \citep{propublica_compas} datasets. The Adult-Income dataset is based on the 1994 US Census data and contains 12 features. A model’s task is to classify whether an individual's income is over $\$50,000$.
COMPAS was collected by ProPublica and contains information about the criminal history of defendants from Broward County for analyzing recidivism in the United States. The processed dataset contains 7 features. A model needs to decide bail based on predicting which of the bail applicants will recidivate in the next two years.
We preprocess both datasets based on a previous analysis where categorical features are binarized \citep{carla_neurips}.
\footnote{\label{footnote_ar} The code for the Actionable Recourse method \citep{ar_fat} requires binary categorical variables.} 
Our black-box model is a neural network model with 2-layers. Please refer to Appendix Tables \ref{tab:supp_main_model2} and \ref{tab:supp_main_logistic} for experiments with various model families and to Appendix \ref{sec:supp_setup} and Table \ref{tab:data_stats} for further experimental details and data statistics.

\vspace{-2pt}
\textbf{Baselines:}
\label{setup:baseline}
We compare our methods COLS and P-COLS with DICE \citep{dice_fat}, FACE-Knn and FACE-Epsilon \citep{face}, Actionable Recourse \citep{ar_fat}, and a Random Search which samples counterfactuals uniformly over the space to obtain the next candidate set at each iteration. Importantly, \textbf{we control for compute across methods} by restricting the number of forward passes to the black-box model, which are needed to decide if a counterfactual produces the desired class. For most big models, this is the rate-limiting step for each method. We set a fixed budget of $5000$ model queries, a set size $|\mathcal{S}| = 10$, and number of cost function sample $M=1000$ for all methods. 
For a description of the objective function and other details of these baselines refer to Appendix \ref{sec:supp_other_objective}, \ref{sec:supp_other_methods}.

\textbf{Measuring Recourse Quality:}
\label{setup:metrics}
To compare with past work, we evaluate methods on distance based metrics like feature diversity, proximity, sparsity and validity. Proximity is defined as ${\textstyle prox(\vx, \mathcal{S}) = 1 - \frac{1}{|\mathcal{S}|}\sum_{i=1}^{|\mathcal{S}|} dist(\vx, \mathcal{S}_i)}$, where $\mathcal{S}_i$ is a counterfactual. Sparsity \citep{dice_fat} is defined as ${\textstyle spar(\vx, \mathcal{S}) = 1 -\frac{1}{|\mathcal{S}|*d}\sum_{i=1}^{|\mathcal{S}|}\sum_{j=1}^{|\vx|} \mathbbm{1}_{\{x_j \neq \mathcal{S}_{ij}\}}}$. Feature diversity \citep{dice_fat} is defined as ${\textstyle div(\mathcal{S}) = \frac{1}{Z}\sum_{i=1}^{|\mathcal{S}|-1}\sum_{j=i+1}^{|\mathcal{S}|} dist(\mathcal{S}_i, \mathcal{S}_j)}$, where $Z$ is the number of terms in the double summation. Validity is defined as ${\textstyle val(Y) = \frac{|\{ \text{unique}~ s_i \in \mathcal{S} ~:~ f(s_i) = +1\}|}{|\mathcal{S}|}}$. 
These metric lie between [0, 1] and higher values are better.

Moreover, we introduce a new cost-based metric which is directly linked to user satisfaction and is computed using each user's true cost function $\mathcal{C}^*_u$. 
First, we define that a user is \emph{satisfied} by a recourse set if the best option in that set achieves a sufficiently low cost under their cost function $\mathcal{C}^*_u$. 
Satisfying users with recourses is what we principally care about.
Hence, given a set of users $\mathcal{U}$ and a set of generated recourse sets $\{\mathcal{S}_{u}\}_{u \in \mathcal{U}}$, we define the fraction of users satisfied at a satisfiability threshold $k$, FS@$k$ as:
\begin{equation} 
    \footnotesize
    FS@k(\mathcal{U}, \{\mathcal{S}_{u}\}_{u \in \mathcal{U}}) =  \frac{1}{|\mathcal{U}|}\sum_{u \in \mathcal{U}}\mathbbm{1}_{\{\text{MinCost}(\vs_{u}, \mathcal{S}_u ; \mathcal{C}^*_u) < k\}}     
    \label{eqn:FS}
\end{equation}
While in reality $k$ will vary from user to user, we keep $k$ fixed across users in our experiments because the goal of any method is to find low-cost recourses regardless of $k$, and small differences to $k$ have no effect on how methods rank-order under FS@$k$. 
We note that minimizing user-incurred cost in Eqn. \ref{eqn:main} is implicitly maximizing the probability of satisfaction of that user, which in turn increases FS@$k$.

As in past work, we also measure Population Average Cost (PAC), which is defined as $\text{PAC} = \frac{1}{|\mathcal{U}|}\sum_{u \in \mathcal{U}}\text{MinCost}(\vs_{u}, \mathcal{S}_u ; \mathcal{C}^*_u)$. But we note that average cost cannot be used to assess individual users' satisfaction as average for a group can be low while most users might have high cost and therefore be dissatisfied. 
\begin{equation} 
    \footnotesize
    Cov(\mathcal{U}, \{\mathcal{S}_{u}\}_{u \in \mathcal{U}}) = \frac{1 }{|\mathcal{U}|}\sum_{u \in \mathcal{U}}\mathbbm{1}_{\{\text{MinCost}(\vs_{u}, \mathcal{S}_u ; \mathcal{C}^*_u) < \infty\}}    
    \label{eqn:cov}
\end{equation}

\textbf{Simulated User Cost Functions:} 
Given that true user cost functions $\mathcal{C}^*_u \sim \mathcal{D}^*$ are not known in practice, for evaluation experiments we simulated them from a distribution $\mathcal{D}_{test}$. There are two phases in the experiments: (1) the recourse generation (training) phase, where we sample cost functions from $\mathcal{D}_{train}$ and optimize for EMC; (2) the evaluation (test) phase, where we use the cost function $\mathcal{C}^*_u$ which is hidden during training to compute cost-based metrics for all the recourse methods.
For experiments apart from distribution shift experiments, $\mathcal{D}_{train} = \mathcal{D}_{test} = \mathcal{D}_{mix}$ which is the most general distribution we propose.
Please note that this is analogous to supervised learning where train and test distributions are assumed to be same, since the cost samples $\{\mathcal{C}_i\}_{i=1}^M$ used in EMC are different from $\mathcal{C}^*_u$ which is used for evaluation (unlike in past works which reuse a cost function at train and test time; see Section \ref{sec:related}).

\begin{table*}[t!]
\footnotesize
\centering
\resizebox{.9\textwidth}{!}{
\begin{tabular}{ccccccccc}
\toprule
\textbf{Data}                       & \textbf{Method} & \multicolumn{7}{c}{\textbf{Metrics}}                                                                                    \\ \midrule
                                 &       & \multicolumn{3}{c}{\textbf{Cost Metrics}}  & \multicolumn{3}{c}{\textbf{Distance Metrics}}  \\ \cmidrule(lr){3-5} \cmidrule(lr){6-8}
                                 
                                 &       & \textbf{FS@1} & \textbf{PAC$(\downarrow)$} & \textbf{Cov} & \textbf{Div} & \textbf{Prox} & \textbf{Spars} & \textbf{Val} \\ \cmidrule(lr){3-9}

\multirow{7}{*}{\textbf{Adult-Income}}  & \textbf{DICE} &	2.47 &		1.37 &	8.32 &	3.90 &	65.80 &	47.20  &	97.90 \\
                                & \textbf{Face-Eps} &	15.23  &	0.76 &	22.60 &	4.75 &	\textbf{92.22} &	74.98  &	\textbf{100.0} \\
                                 & \textbf{Face-Knn} &	25.30  &	0.74 &	35.00 &	8.62 &	89.07 &	71.85  &	\textbf{100.0} \\
                                 & \textbf{Act. Recourse} &	49.93 &	0.55 &	56.85 &	18.38 &	74.68 &	73.57  &	78.67 \\
                                 & \textbf{Random} &	6.27 &	1.40 &	31.83 &	\textbf{48.30} &	55.83 &	39.85 &	95.55 \\ \cmidrule{2-9}
                                 & \textbf{COLS} &	72.57  &	\textbf{0.38} &	76.07 &	25.77 &	80.22 &	76.48  &	97.15 \\
                                 & \textbf{P-COLS} &	\textbf{75.82} & 0.40 &	\textbf{79.20} &	25.57 &	81.67 &	\textbf{78.00} &	94.78 \\
 \midrule
\multirow{7}{*}{\textbf{COMPAS}} & \textbf{DICE} &	0.40  &	0.54 &	0.40 &	11.30 &	65.00 &	32.00  &	98.90 \\
                                & \textbf{Face-Eps} &	12.20  &	0.29 &	12.20 &	2.50 &	\textbf{94.20} &	60.60 &	\textbf{100.0} \\
                                & \textbf{Face-Knn} &	12.20  &	0.29 &	12.20 &	2.60 &	94.10 &	60.60 &	\textbf{100.0} \\
                                & \textbf{Act. Recourse} &	65.80  &	0.40 &	66.60 &	11.87 &	80.53 &	\textbf{74.07} &	44.23 \\
                                & \textbf{Random} &	29.95  &	0.77 &	39.20 &	\textbf{42.22} &	55.90 &	31.25 &	71.88 \\  \cmidrule{2-9}
                                & \textbf{COLS} &	82.23 &	\textbf{0.24} &	82.23 &	29.32 &	77.82 &	70.05 &	95.48 \\
                                & \textbf{P-COLS} &	\textbf{83.73}  &	\textbf{0.24} &	\textbf{83.73} &	29.38 &	78.48 &	71.30 &	92.78 \\

\bottomrule
\end{tabular}
}
\caption{\label{tab:main} Recourse method performance across various cost and distance metrics (Section \ref{setup:metrics}). The numbers reported are averaged across $5$ different runs.
For all the metrics higher is better except for PAC where lower is better. Refer to section \ref{exp:main} for more details.\vspace{-10pt}}
\end{table*}

\begin{table}[t]
\footnotesize
\centering
\resizebox{0.9\columnwidth}{!}{
\begin{tabular}{ccccccccc}
\toprule

\textbf{Search Alg.} & \textbf{Objective} & \multicolumn{3}{c}{\textbf{Cost Metrics}}  & \multicolumn{3}{c}{\textbf{Distance Metrics}}  \\ \cmidrule(lr){3-5} \cmidrule(lr){6-8}

 & & \textbf{FS@1}  & \textbf{PAC$(\downarrow)$} & \textbf{Cov} & \textbf{Div} & \textbf{Prox} & \textbf{Spars} \\ \midrule

\textbf{LS} & \textbf{Sparsity} & 10.1      & 1.304 & 29.0  & 42.7 & 66.2  & 55.8  \\
\textbf{LS} & \textbf{Proximity}    & 9.7  & 1.275 & 27.0 & 42.1 & 67.5 & 55.0\\
\textbf{LS} & \textbf{Diversity}    & 0.0     & 2.393 & 7.6   & \textbf{53.3} & 50.8  & 35.6  \\ \midrule
\textbf{LS} & \textbf{EMC}  & 49.8  & 0.597 & 59.1  & 37.8  & 73.3  & 67.5  \\
\textbf{COLS} & \textbf{EMC}  & \textbf{68.8} & \textbf{0.391}    & \textbf{72.6} & 27.1  & \textbf{77.5} & \textbf{73.5} \\

\bottomrule 
\end{tabular}
}
\caption{\label{tab:ablation}Ablation results with Search algorithms trained on different objectives (Section \ref{exp:ablation}).
\vspace{-10pt}}
\end{table}

\subsection{Research Questions}
\label{sec:questions}
\textbf{Q1. Which Method Satisfies the Most Users?}
\label{exp:main}

In this experiment, we compare recourse methods from Section \ref{sec:setup} and Appendix \ref{sec:supp_other_methods} on distance and cost-based metrics (Section \ref{sec:setup}). We perform five runs for each method with different random seeds, reporting the mean results in Table \ref{tab:main}. We omit the variances as they were less than 0.01.
\vspace{3pt}
\\\textbf{Results:} We observe that DICE, which optimizes for a combination of distance-based metrics, performs much worse on metrics like Population Coverage (Cov) and FS which directly model user-incurred cost and satisfaction. Meanwhile, \textbf{COLS and P-COLS, which optimize for EMC, achieve 22.64\% and 25.89\% point higher user satisfaction while covering 19.28\% and 22.42\% point more users} on Adult-Income and COMPAS, respectively. 
We also observe that \textbf{COLS and P-COLS demonstrate high sparsity and proximity in the solutions}. 
Interestingly, we find that COLS and P-COLS do not exhibit high feature diversity. 
The second-best method on cost metrics, Actionable Recourse, also promotes proximate and sparse solutions rather than feature diversity, while Random Search achieves high diversity but low user satisfaction.
These results provide evidence that feature diversity among recourses might not be a necessary condition for high user satisfaction.
Hence, it is preferable to provide users with recourse based on expected cost as opposed to providing them diverse options which might not align well with their preferences.

\textbf{Q2. Is the Performance Improved by the COLS Optimization Method or by the EMC Objective?}
\label{exp:ablation}

We perform an ablation study to attribute the improvements from our method to either the COLS optimization method or the EMC objective function. To do so, we run a basic local search (LS) to optimize objectives used by other methods like feature diversity, proximity, and sparsity along with validity. We use a basic non-optimized local search, because there is no simple and efficient way to guarantee a reduction in the diversity objective by swapping out single elements from the solution set which is required for COLS (see Appendix \ref{sec:supp_method}).
We also optimize for EMC using basic local search to understand the usefulness of COLS.
\vspace{3pt}
\\\textbf{Results:}  The results in Table \ref{tab:ablation} suggest that optimizing for metrics besides EMC is sub-optimal. For proximity, sparsity, and feature diversity objectives, the FS score and population coverage is very low, while they perform well on their respective metrics. The low FS score for distance metrics is expected as they ignore user-specific preference while optimizing for their objectives, and therefore the generated recourses might be infeasible under the user's cost function. We find that EMC with LS outperforms all distance metrics on FS, which suggests that the EMC is a better objective and leads to higher user satisfaction.
Meanwhile, the \textbf{19\% point difference in the performance of EMC with LS and COLS can be attributed to our cost optimization}
described in Section \ref{sec:search_method} and Theorem  \ref{thm:monotonicity}, which allows COLS to more efficiently search the solution space. 
We observe that \textbf{sparsity and proximity are positively correlated with higher user satisfaction} and emerge from the idea of \textit{Linear} and \textit{Percentile} costs.

\textbf{Q3. Are Recourses Fair Across Subgroups?}
\label{exp:fairness}

\textbf{Design:} We want to understand whether recourse methods provide equitable solutions across subgroups based on demographic features like \textit{Gender}. We adapt existing fairness metrics for disparate impact across population subgroups \citep{di_fairness_metric} for the recourse outcomes we study, which we denote by Disparate Impact Ratio (DIR). Given a metric $\mathcal{M}$, DIR is a ratio between metric scores across two subgroups. $\text{DIR-}\mathcal{M} = \mathcal{M}\text{(S=1)}/\mathcal{M}\text{(S=0)}$. We use either Cov or FS@1 as $\mathcal{M}$. 
Under the DIR metric, the maximum fairness score that can be achieved is 1, though this might not be achievable depending on the black-box model. 
We run experiments on the Adult-Income dataset, with a budget of $5000$ model queries and $|\mathcal{S}| = 10$. 
\vspace{3pt}
\\\textbf{Results:} We present the gender based subgroup results in Table \ref{tab:disparity}. For results on racial based subgroup see Appendix Table \ref{tab:race_disparity}.
We observe that \textbf{our methods are typically more fair than baselines on both Gender and Race-based subgroups while providing recourse to a larger fraction of people in both subgroups}.
In particular, we find that our method achieves a score very close to $1$ on DIR-FS and DIR-Cov implying a very high degree of fairness.
We attribute the fairness of our method to (1) the fact that our method does not heavily depend on the data distribution, and (2) the use of a diverse set of cost functions when generating recourses.
We see condition (2) as important since there are other individualized methods that do not rely on the data distribution, such as Actionable Recourse, which can generate less fair solutions than COLS.

\begin{table}[t]
\small
\centering
\resizebox{0.75\columnwidth}{!}{
\begin{tabular}{c|ccccc}
\toprule
\textbf{Method} & \textbf{Gender} & \textbf{FS@1} & \textbf{Cov} & \textbf{DIR-FS} & \textbf{DIR-Cov} \\
\midrule

\multirow{2}{*}{\textbf{DICE}} & \textbf{F} & 0.0 &  0.0 & \multirow{2}{*}{-} & \multirow{2}{*}{-}  \\
& \textbf{M} & 4.7 &  15.6 & & \\ \midrule

\multirow{2}{*}{\textbf{Face-Eps}} & \textbf{F} & 12.5 &  22.1 & \multirow{2}{*}{1.504} & \multirow{2}{*}{1.118}  \\
& \textbf{M} & 18.8 &  24.7 & &  \\ \midrule

\multirow{2}{*}{\textbf{Face-Knn}} & \textbf{F} & 29.9 &  36.3 & \multirow{2}{*}{0.719} & \multirow{2}{*}{0.89}  \\
& \textbf{M} & 21.5 &  32.3 & & \\ \midrule

\multirow{2}{*}{\textbf{Act. Recourse}} & \textbf{F} & 53.8 &  58.7 & \multirow{2}{*}{0.881} & \multirow{2}{*}{0.959} \\
& \textbf{M} & 47.4 &  56.3 & & \\ \midrule

\multirow{2}{*}{\textbf{Random}} & \textbf{F} & 7.8 &  34.6 & \multirow{2}{*}{0.859} & \multirow{2}{*}{0.792}  \\
& \textbf{M} & 6.7 &  27.4 & & \\ \midrule

\multirow{2}{*}{\textbf{COLS}} & \textbf{F} & 72.7 &  76.2 & \multirow{2}{*}{0.994} & \multirow{2}{*}{0.992}  \\
& \textbf{M} & 72.3 &  75.6 & & \\ \midrule

\multirow{2}{*}{\textbf{P-COLS}} & \textbf{F} & 76.5 &  80.2 & \multirow{2}{*}{\textbf{1.004}} & \multirow{2}{*}{\textbf{1.0}} \\
& \textbf{M} & 76.8 &  80.2 & &  \\
\bottomrule
\end{tabular}
}

\caption{\label{tab:disparity}Fairness analysis of recourse methods for subgroups with respect to Gender and Race.
\textbf{DIR}: Disparate Impact Ratio; \textbf{M}: Male, \textbf{F}: Female (Section \ref{exp:fairness}).\vspace{-5pt}}
\end{table}

\textbf{Q4. Which Method Do Humans Prefer?}
\label{exp:human_eval}

While we can compute our linear and percentile cost functions for users, we are especially interested in whether humans would consider recourses to reasonable for our synthetic users.
We designed a small study where we provided human annotators with state vectors and preference scores $\vp$ for a sample of 100 users in the Adult-Income dataset (see Appendix \ref{sec:supp_human} for more details).
We presented the recourse generated by COLS and Actionable Recourse (strongest baseline) to the annotators while anonymizing each method's name and asked them two questions: (1) Acting as if they were the user with the provided preferences and state vector, which recourse would they prefer to adopt? (2) Does the recourse generated by each method seem reasonable to them? We collect three annotations for each sample and take a majority vote for each response; we allow for users to indicate ``no preference" between the two proposed recourses, and if there was a three-way tie in annotation we record the majority vote as ``no preference."

\begin{table}[h]
    \centering
    \begin{tabular}{c|c|c}
    \toprule
        No Preference & Actionable Recourse & Ours  \\
        \midrule
        18\% & 25\% & \textbf{57\%} \\
        \bottomrule
    \end{tabular}
    \caption{Percentage of times each method was preferred by human annotators (Fleiss kappa=0.74 and $p{=}$1e-4).\vspace{-7pt}}
    \label{tab:my_label}
\end{table}

\textbf{Results:} We found that \textbf{our method was preferred 57\% of the time, while AR was preferred only 25\% of the time}, a difference of 32 percentage points (+/- 16 points variance, Fleiss' kappa=0.74, and $p{=}$1e-4). Furthermore, human annotators found 60\% of the recourses generated by COLS to be reasonable as compared to 33\% for AR, a 27 point difference ($p{<}$1e-4). This study shows that our method is preferred by humans over the baseline.
\\ 
\\
\\
\textbf{Q5. Robustness to Distribution Shifts?}

In this experiment, we measure the effect of distribution shift between the train and test time distributions, $\mathcal{D}_{train}$ and $\mathcal{D}_{test}$. The top-left and bottom-right corners of Appendix Figure \ref{fig:alphagrid} show that our method is robust to deviations where the costs for EMC are obtained from $\mathcal{D}_{lin}$ while the true user costs are drawn from $\mathcal{D}_{per}$ (and vice versa), which is a complete distribution shift. For full experimental design and conclusions please refer to Appendix Figures \ref{fig:alphagrid} and \ref{fig:ds}.

We provide experiments for several additional research questions in the Appendix \ref{sec:supp_questions}, which we summarize here:
(1) We can make use of a larger compute budget to scale up the performance (Figure \ref{fig:budget}); (2) The recourse sets provide \textit{high quality solutions to users using as few as 3 counterfactuals} (Figure \ref{fig:num_cfs}); and 
(3) we can \textit{achieve high user satisfaction with as few as 20 Monte Carlo samples}, rather than 1000 (Figure \ref{fig:num_mcmc}). We also show some qualitative examples of recourses provided by our method in Table \ref{tab:examples}.

\section{Conclusion}
In this paper, we propose a cost-based recourse generation method which can optimize for an unknown user-specific cost function. We show that our method achieves much higher rates of user satisfaction than comparable baselines. This is particularly useful since detailed user cost function data is not readily available for most applications.
We attribute performance gains to our Expected Minimum Cost (EMC) objective term, and we also show that our discrete optimization algorithm, Cost-Optimized Local Search (COLS), produces large improvements over baseline search methods. 

\section{Ethics Statement}

We hope that our recourse method is adopted by institutions seeking to provide reasonable paths to users for achieving more favorable outcomes under the decisions of black-box machine learning models or other inscrutable models. We see this as a ``robust good," 
similar to past commentators \cite{venkatasubramanian2020philosophical}. Below, we comment on a few other ethical aspects of the algorithmic recourse problem. 

First, we suggest that fairness is an important value which recourse methods should always be evaluated along, but we note that evaluations will depend heavily on the model, training algorithm, and training data. For instance, a sufficiently biased model might not even allow for suitable recourses for certain subgroups. As a result, any recourse method will fail to identify an equitable set of solutions for the population. That said, recourse methods can still be designed to be more or less fair. This much is evident from our varying results on fairness metrics using a number of recourse methods. What will be valuable in future work is to design experiments that separate the effects on fairness of the model, training algorithm, training data, and recourse algorithm. Until then, we risk blaming the recourse algorithm for the bias of a model, or vice versa.

Additionally, there are possible dual-use risks from developing stronger recourse methods. For instance, malicious actors may use recourse methods when developing models in order to \emph{exclude} certain groups from having available recourse, which is essentially a reversal of the objective of training models for which recourse is guaranteed \citep{Ross2020LearningMF}. We view this use case as generally unlikely, but pernicious outcomes are possible. We also note that these kinds of outcomes may be difficult to detect, and actors may make bad-faith arguments about the fairness of their deployed models based on other notions of fairness (like whether or not a model has access to protected demographic features) that distract from an underlying problem in the fairness of recourses.

\section{Reproducibility Statement}
To encourage reproducibility, we provide our source code, including all the data pre-processing, model training, recourse generation, and evaluation metric scripts at \url{https://github.com/prateeky2806/EMC-COLS-recourse}. The details about the datasets and the pre-processing is provided in Appendix \ref{sec:supp_datasets}. We also provide clear and concise Algorithms \ref{alg:sampling}, \ref{alg:percentile}, \ref{alg:linear} for our cost sampling procedures and our optimization method COLS in Algorithm \ref{alg:cols}.

\section*{Acknowledgements}
We thank Xiang Zhou, Shiyue Zhang, Yixin Nie, Adyasha Maharana, Swarnadeep Saha, Archiki Prasad and Jaemin Cho for feedback on this paper. 
This work was supported by NSF-CAREER Award 1846185, DARPA Machine-Commonsense (MCS) Grant N66001-19-2-4031, Royster Society PhD Fellowship, Microsoft Investigator Fellowship, and Google and AWS cloud compute awards. The views contained in this article are those of the authors and not of the funding agency.

\bibliography{iclr2022_bib}
\bibliographystyle{iclr2022_bst}

\newpage
\appendix

\section{Appendix - Additional Experiments and Details}

\subsection{Experimental Setup}
\label{sec:supp_setup}

\begin{table}[t]
\small
\centering
\resizebox{.85\textwidth}{!}{
\begin{tabular}{l|cccc}
\toprule

 & Adult-Income Binary & COMPAS Binary & Adult-Income & COMPAS \\ \midrule
\# Continuous features & 3 & 4 & 2 & 3\\
\# Categorical features & 9 & 3 & 10 & 12 \\
Undesired class & $\leq$ 50k & Will Recidivate & $\leq$ 50k  & Will Recidivate \\
Desired class & $>$ 50k & Won't Recidivate & $>$ 50k  & Won't Recidivate \\
Train/val/test & 20088/2338/749 & 1415/229/491 & 13172/1569/748 & 5491/705/444 \\
Model Type & ANN(2, 20) & ANN(2, 20) & ANN(2, 20) & ANN(2, 20) \\
Val Accuracy & 82\% & 69\% & 81\% & 61\% \\

\bottomrule 
\end{tabular}
}
\caption{\label{tab:data_stats}Table containing data statistics and black-box model details. The binary version of the datasets are take from \citep{carla_neurips} whereas the non-binary version are taken from \citep{dice_fat}.}
\end{table}

\subsubsection{Datasets and Black-Box Model} 
\label{sec:supp_datasets}

In our experiments, we have two versions of the dataset, one with binary categorical features, whereas the other with non-binary categorical features. In the main paper, we show results on the binarized version (Table \ref{tab:main}) as an important baseline, Actionable Recourse \citep{ar_fat}, operates with binary categorical features.\footnote{The binary datasets can be downloaded from \href{https://github.com/carla-recourse/cf-data}{https://github.com/carla-recourse/cf-data}, whereas the non-binary data can be found on \href{https://github.com/interpretml/DiCE}{https://github.com/interpretml/DiCE}.} The data statistics for all the datasets can be found in Table \ref{tab:data_stats}. In our experiments, for all the datasets, the features gender and race are considered to be immutable \citep{dice_fat}, since we perform subgroup analysis with these variables that would be rendered meaningless if users could switch subgroups. 
Other features can either be mutable or conditionally mutable depending on semantics. These constraints can be incorporated into the methods by providing a schema of feature mutability criterion. Our black-box model is a multi-layer perceptron model with 2 hidden layers trained on the trained set and validated on the dev set. The accuracy numbers are shown in Table \ref{tab:data_stats}. The test set which is used in the counterfactual generation experiments only contains users which are classified to the undesired class by the trained black-box model. Note that our frameworks can operate with any type of model, the only requirement is the ability to query the model for outcome given a user's state vector.

\subsubsection{Computational Complexity:} 

COLS is a local search-based method and runs for $\mathcal{O}(\frac{B}{|S|})$ iterations for each user to generate the recourse set, where B is the budget (see section 5.1 - Baselines). The complexity of the cost optimization step in COLS is $\mathcal{O}(|S|^2 *M)$ per iteration. Values of $|S|$ and M as low as 3 and 10 respectively work well in practice (see Appendix B.2 and Figure \ref{fig:num_cfs}, \ref{fig:num_mcmc}). Finally for the current implementation the wall clock time on the adult dataset for each user with $|S|$ = 10, M = 100, B = 5000 setting is COLS = 20s, Random = 7.5s, DICE = 7.5s, AR = 11s, Face-knn = 7s, Face-Eps = 6s (can be parallelized across users). Cost function samples can be pre-computed once and saved for all experiments, this typically takes a few minutes ($<$ 5 min) across all users.

\begin{figure}[t]
    \centering
    \begin{minipage}{0.47\textwidth}
        \centering
        \includegraphics[width=\textwidth]{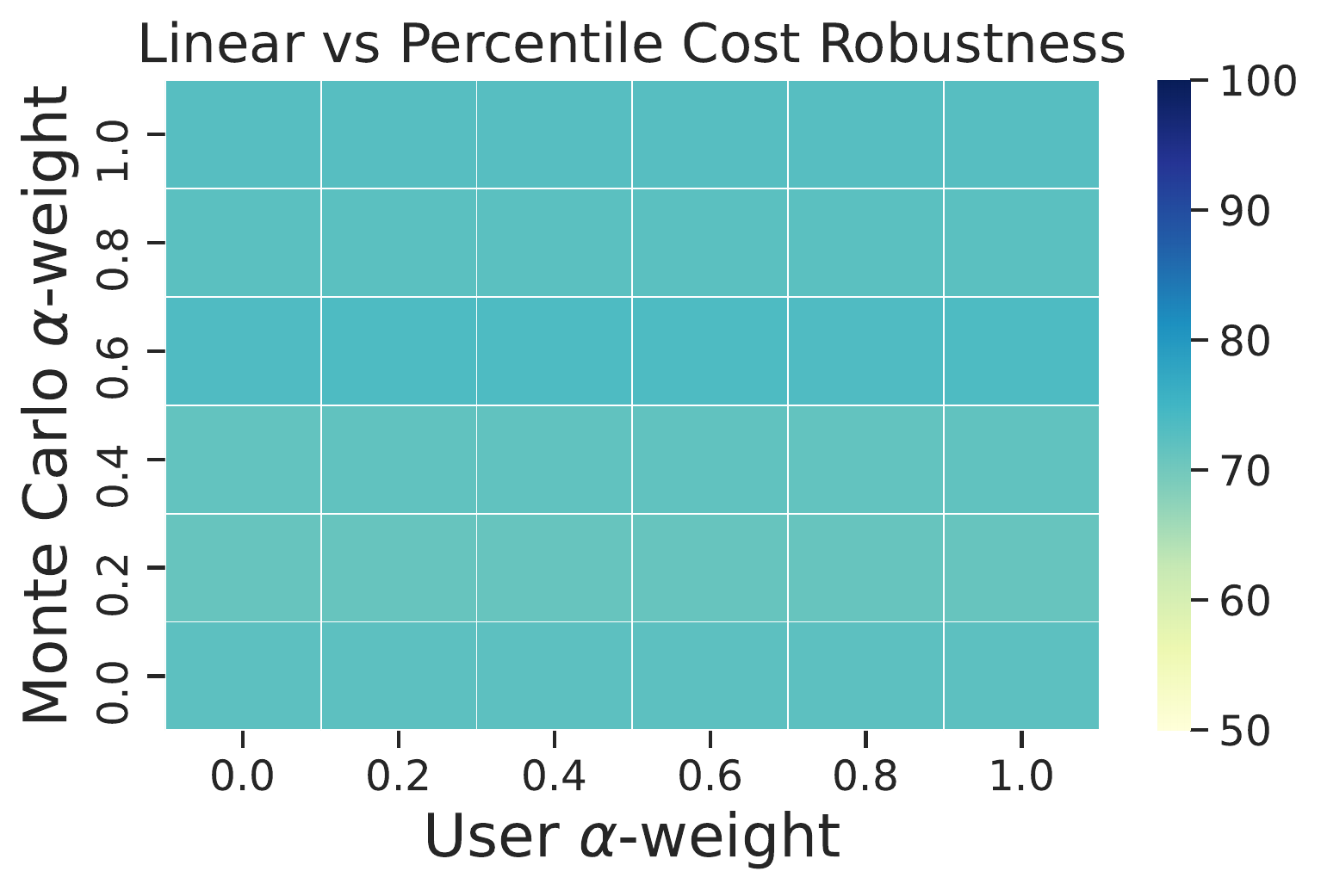}
        \caption{\label{fig:alphagrid}This figure shows the performance of the method on FS@$k$ when recourses are generated with Monte Carlo cost samples from a distribution with $\alpha$-weight varying between 0 and 1, where the user costs follow different $\alpha$-weight values varying between 0, 1. Performance is robust to misspecification of $\alpha$. Refer to Section \ref{exp:alpha_robust} for more details.}
    \end{minipage}
     \begin{minipage}{.47\textwidth}
        \centering
        \includegraphics[width=\textwidth]{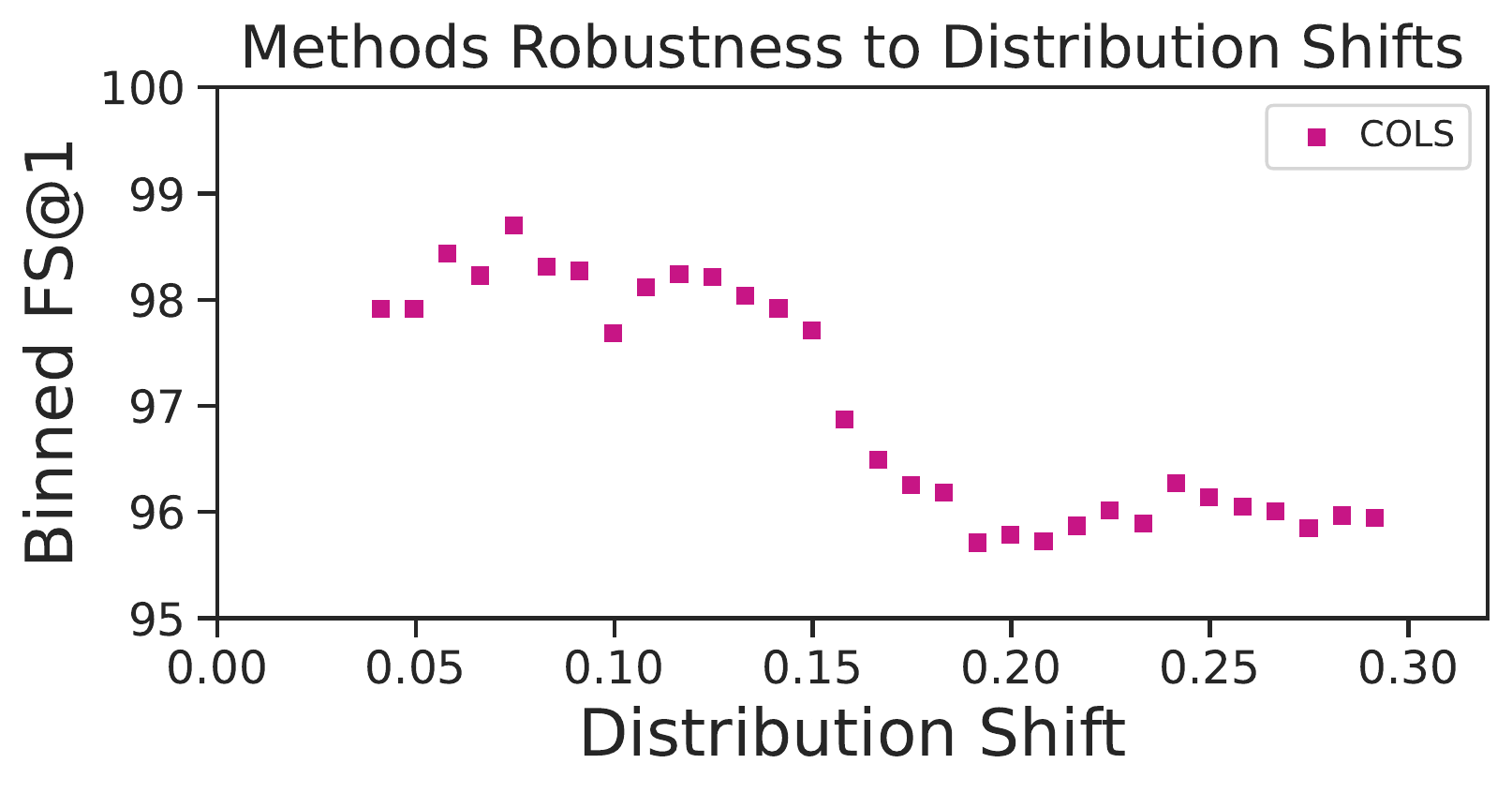}
        \caption{\label{fig:ds}In this plot we show the fraction of users satisfied vs the distance between the train and test distributions. The results demonstrate that as the distance increases the performance drops a bit and then plateaus, which means that the method is robust to this kind of distribution shift. Please refer to Section \ref{exp:ds} for more details. }
    \end{minipage}\hfill
    
\end{figure}

\subsubsection{Recourse Generation and Evaluation Pipeline}
\label{sec:supp_design}
To approximate the expectation in equation \ref{eqn:obj}, our algorithm samples a set of random cost functions $\{\mathcal{C}_i\}_{i=1}^M \sim \mathcal{D}_{train}$, which are used at the generation time to optimize for the user's hidden cost function. In the generation phase, we use Equation \ref{eqn:mcmc} as our objective. Note that, this objective promotes that the generated counterfactual set contains at least one good counterfactual for each of the cost samples, hence this set satisfies a large variety of samples from $\mathcal{D}_{train}$. This is achieved via minimizing the mean of the minimum cost incurred for each of the Monte Carlo samples \citep{monte_carlo}. Equivalently, the objective is minimized by a set of counterfactuals $\mathcal{S}$ where for each cost function there exists an element in $\mathcal{S}$ which incurs the least possible cost. In practice the size of set $\mathcal{S}$ is restricted, hence we may not achieve the absolute minimum cost but the objective tries to ensure that the counterfactuals which belong to the set have a low cost at least with respect to one Monte Carlo cost sample. The generation phase outputs a set of counterfactuals $\mathcal{S}$ which is to be provided to the users as recourse options. Given this set $\mathcal{S}_u$, in the evaluation phase, we use the users simulated cost functions which are hidden in the generation phase, to compute the cost incurred by the user $\text{MinCost}(\vs_{u}, \mathcal{S} ; \mathcal{C}^*_u)$ and calculate the metrics defined in the Section \ref{sec:setup}.

\begin{remark}
    In the limiting case, as $|\mathcal{S}| \rightarrow M$ and $\text{Budget} \rightarrow \infty$,  RHS of equation \ref{eqn:mcmc}; $\min\limits_{\mathcal{S}} \text{MinCost}(\vs_u, \mathcal{S} ; \mathcal{C}_i) \rightarrow \text{optimum cost}$, $\forall ~ i \in [M]$, if feasible solutions are possible for all cost functions samples $\{\mathcal{C}_i\}_{i=1}^M$. 
    \label{rem:limit}
\end{remark}
What this remark implies is that, as the set size increases to the number of Monte Carlo samples $M$, then the methods can trivially generate one counterfactual for each cost function sample which minimizes its cost. In that limiting case, we get custom counterfactual for each cost function. Hence, in the limiting case we can find optimum solution. In the experiment \ref{exp:cfs}, we show that even with very less number of counterfactual we are able to achieve good performance implying that in practice we do not need large set size.

\subsubsection{Details of Human Evaluation}
\label{sec:supp_human}
For our human evaluation experiments, we had three undergraduate research assistants with a background in computer science. They were provided with a set of instructions on how to interpret and perform the task. Specifically, in virtual meetings, we provided them with an overview of the dataset along with the feature descriptions, a description of the task, and an overview of the recourse generation problem. Before testing, we conducted a small understanding quiz including example problems, and we corrected any misunderstandings of the study procedure.
For each data point, they were asked to assume that they were a hypothetical user with the given state vector and preference scores in the sample and then were provided with the recourses generated by our method and Actionable Recourse \cite{ar_fat} (in a blind format with randomized ordering). In total, we collected three annotations each for 100 samples from the Adult-Income dataset.

\subsection{Additional Research Questions}
\label{sec:supp_questions}

\textbf{Q5. Robustness to Misspecification in population's true and proposed cost function distribution?}
\label{exp:alpha_robust}

\textbf{Design:} Our $\mathcal{D}_{mix}$ distribution samples cost by taking an $\alpha$-weighted combination of linear and percentile costs. These two cost have different underlying assumptions about the how users view the cost of transition between the states. We want to test the robustness of our method in terms of misspecification in users disposition to these types of cost. We perform a robustness analysis where the users cost function has a different $\alpha$ mixing weight as compared to the Monte Carlo samples we use to optimize for EMC. This creates a distribution shift in the user cost function distribution ($\mathcal{D}_{test}$) and the Monte Carlo sampling distribution ($\mathcal{D}_{train}$) used in EMC. We vary the user and Monte Carlo distributions $\alpha$-weights within the range of 0 to 1 in steps of 0.2. At the extremes values of $\alpha = 0, 1$, the shifts are very drastic as the underlying distribution changes completely. In the case when monte carlo $\alpha$ weight is 0 and user $\alpha$ weight is 1 then $\mathcal{D}_{train} = \mathcal{D}_{perc}$ and $\mathcal{D}_{test} = \mathcal{D}_{lin}$, simlarly for the other case we get $\mathcal{D}_{train} = \mathcal{D}_{lin}$ and $\mathcal{D}_{test} = \mathcal{D}_{perc}$. Please note that the distribution $\mathcal{D}_{lin}$ and $\mathcal{D}_{perc}$ have completely different underlying principles and are two completely different distributions. Hence, the corners of the heatmap represent drastic distribution shifts. 

\textbf{Results:} In Figure \ref{fig:alphagrid}, we show a heatmap plot to which demonstrates the robustness of our method. The color of the block corresponding to Monte Carlo alpha, $\alpha_{mc} = x$ and the users alpha, $\alpha_{user} = y$ represents the fraction of users that were satisfied when $\alpha_{mc} = x$ and $\alpha_{user} = y$. This means that if the user thought of costs only in terms of Linear step involved but the recourse method used samples with only percentile based cost, still the recourse set can satisfy almost the same number of users.
In Figure 3, the corners correspond to these extreme cases described above, the user satisfaction for the top left corner ($\mathcal{D}_{train} = \mathcal{D}_{perc}$ and $\mathcal{D}_{test} = \mathcal{D}_{lin}$) is similar to the bottom left corner ($\mathcal{D}_{train} = \mathcal{D}_{lin}$ and $\mathcal{D}_{test} = \mathcal{D}_{lin}$). Similarly things happen for the opposite case which is denoted by the top-right ($\mathcal{D}_{train} = \mathcal{D}_{perc}$ and $\mathcal{D}_{test} = \mathcal{D}_{perc}$) and bottom-right ($\mathcal{D}_{train} = \mathcal{D}_{lin}$ and $\mathcal{D}_{test} = \mathcal{D}_{perc}$) corners. This means that even when a complete distribution shift occurs the performance user satisfaction remains similar. This can be attributed to the hierarchical step for user preference sampling in the procedure because the preferences values can be arbitrary and they scale the raw percentile and linear cost hence the distribution designed this way to model extremely diverse types of transition costs. 

This means that \textbf{our methods is robust to misspecification in the train and test distributions.} The almost consistent color of the grid means that \textbf{there is very slight variation in the Fraction of Satisfied users when the model is tested on out of distribution user cost types.} \\

\begin{figure}[t]
    \centering
    \begin{minipage}{0.47\textwidth}
        \centering
        \includegraphics[width=\columnwidth]{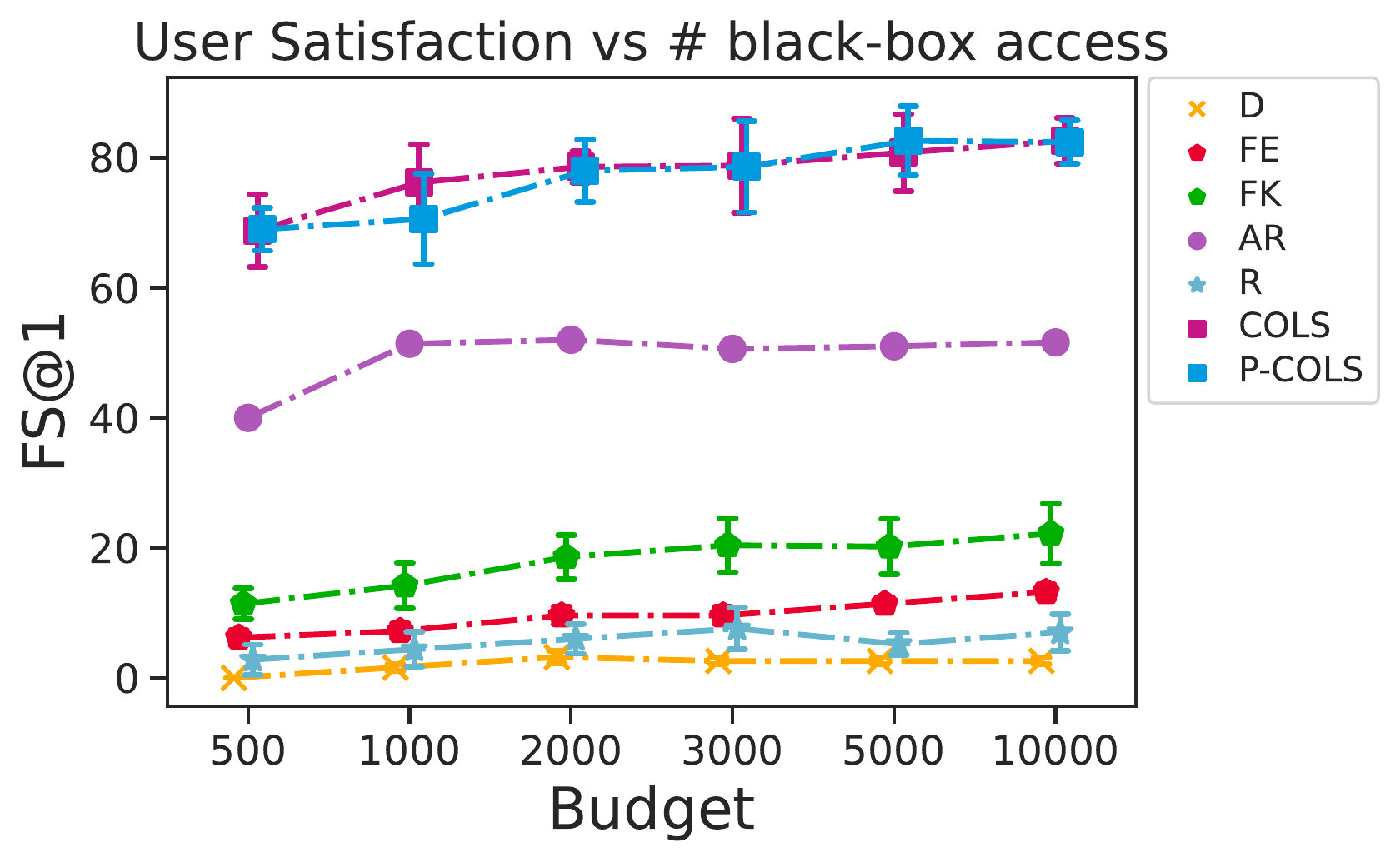}
        \caption{\label{fig:budget}Figure showing the performance of different recourse methods as the Budget is increased. These are the average number across 5 different runs along with the standard deviation error bars. For some methods the standard deviation is very low hence not visible as bars in the plot. It can be seen that as the budget increases the performance of COLS and P-COLS increases. Please refer to Section \ref{exp:budget} for more details.}
    \end{minipage}\hfill
    \begin{minipage}{0.47\textwidth}
        \centering
        \includegraphics[width=\columnwidth]{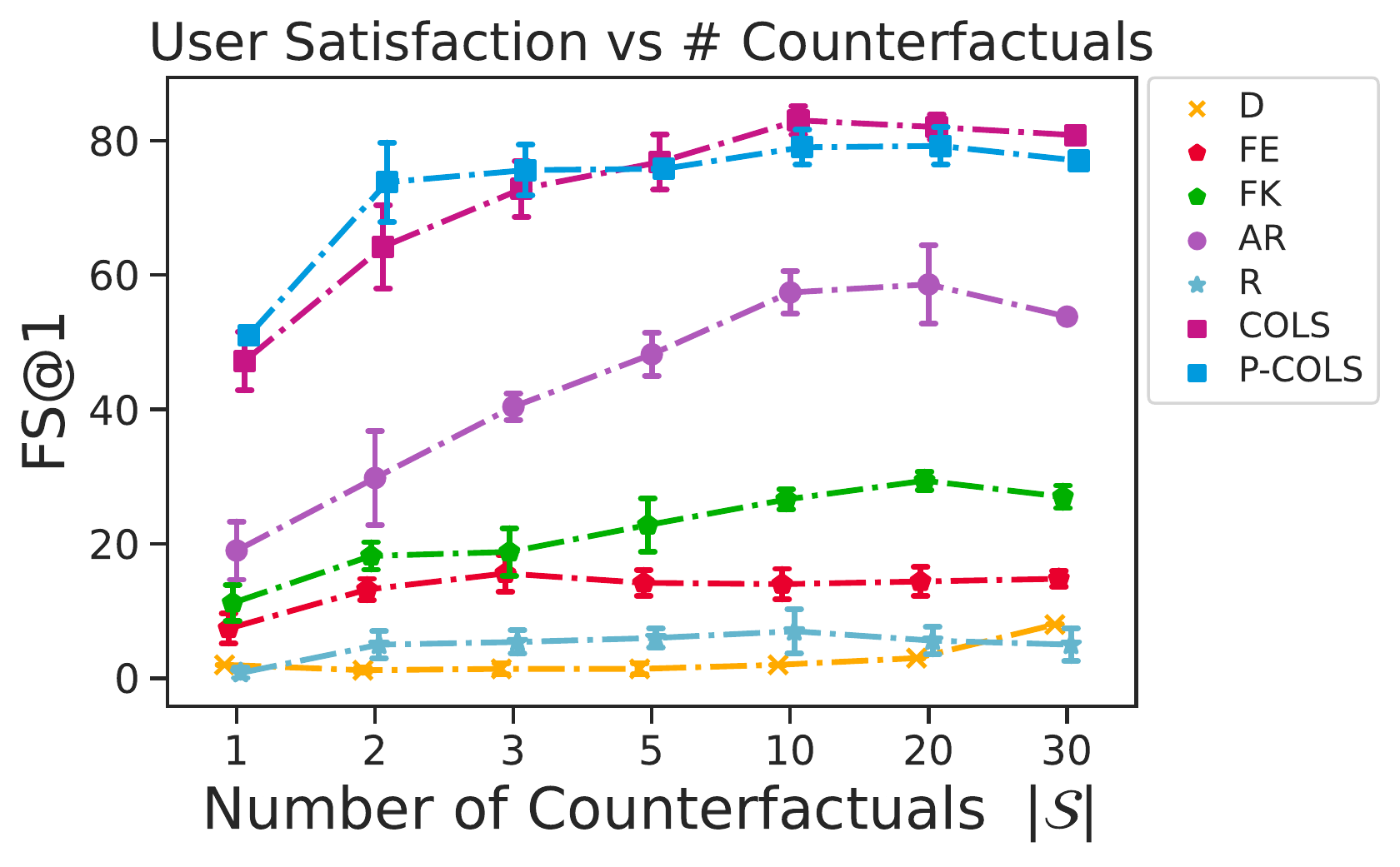}
        \caption{\label{fig:num_cfs}Figure showing the performance of different recourse methods as the the number of counterfacuals to be generated is increased. These are the average number across 5 different runs along with standard deviation error bars. We see that there is a monotonic increase in the fraction of users satisfied as the size of the set increases. We also observe that most of the performance can be obtained with a small set size. Please refer to Section \ref{exp:cfs} for more details.}
    \end{minipage}
\end{figure}

\textbf{Q6. Are Solutions Robust to Misspecified Cost Distributions?}
\label{exp:ds}

\textbf{Design:} In our cost sampling procedure, we make minimal assumptions about the user's feature preferences if they are not provided by the user. When finding recourses, we select a random subset of features along with their preference score for each user. 
However, there are situations where user preferences may be relatively homogeneous for certain features where people usually share common preferences. For example, to increase their \textit{income}, many users might prefer to edit their \textit{occupation type} or \textit{education level} rather than their \textit{work hours} or \textit{marital status}. 
Given the possibility of this kind of distribution shift in feature preferences, we want to measure how robust our method is to distribution shift between our sampling distribution and the actual cost distribution followed by users.

In this experiment, we test a case of this kind of distribution shift over cost functions.
For users in the Adult-Income data, we generate recourse sets using Monte Carlo samples from our standard distribution $\mathcal{D}_{mix}$ (Algorithm \ref{alg:sampling}). 
To obtain hidden user cost functions that differ from this distribution, we first generate $500$ different feature subsets indicating which features are editable, where each subset corresponds to a binary vector $concentration$ representing a user having specific preferences for some features over others (see Sec. \ref{sec:sampling} and Alg. \ref{alg:sampling}). Since having different editable features induces a different distribution over cost functions, we obtain a measure of distribution shift for each of the $500$ $concentration$ vectors by taking an $l_2$ distance between the vector and \textit{its nearest neighbor in the space of {\normalfont{concentration}} vectors used to generate the recourses}. 
We use the nearest neighbor because the most outlying $concentration$ vectors are least likely to be satisfied by the recourse set. In other words, the likelihood that a user is satisfied depends on the minimum distance between their $concentration$ vector and its nearest neighbor in the cost samples used at recourse generation time. Therefore, when the minimum distance increases, there is a greater distribution shift between the user's cost functions and those obtained from $\mathcal{D}_{mix}$.
Finally, we measure how many users are satisfied for a given degree of distribution shift.

\textbf{Results:} In Figure \ref{fig:ds}, we show a binned plot of FS@1 against our measure of distribution shift. We observe that as the distance between the distributions increases, the fraction of users satisfied decreases slightly and then plateaus. Even at the maximum distance we obtain, performance has only dropped about 3 points. This implies that \textbf{our method is robust to distribution shift in the cost distribution in terms of which features people prefer to edit.}
We attribute this result to the fact that our method (1) assumes random feature preferences which subsumes these skewed preferences and (2) provides multiple recourse options, each of which can cater to different kinds of preferences. As a result, we achieve a good covering of the cost function space (see experiments with respect to varying recourse set size and number of sampled cost functions in the Appendix \ref{sec:supp_questions}).

\begin{minipage}{\textwidth}
\begin{minipage}[b]{0.49\textwidth}
\small
\centering
\resizebox{1\columnwidth}{!}{
\begin{tabular}{c|ccccc}
\toprule
\textbf{Method} & \textbf{Race} & \textbf{FS@1} & \textbf{Cov} & \textbf{DIR-FS} & \textbf{DIR-Cov} \\
\midrule

\multirow{2}{*}{\textbf{DICE}} & \textbf{NW} & 0.0 & 0.0 & \multirow{2}{*}{-} & \multirow{2}{*}{-} \\
& \textbf{W} & 3.1 & 10.4 & & \\ \midrule

\multirow{2}{*}{\textbf{Face-Eps}} & \textbf{NW} & 7.7 & 12.7 & \multirow{2}{*}{2.312} & \multirow{2}{*}{2.047} \\
& \textbf{W} & 17.8 & 26.0 & & \\ \midrule

\multirow{2}{*}{\textbf{Face-Knn}} & \textbf{NW} & 12.7 & 25.4 & \multirow{2}{*}{2.228} & \multirow{2}{*}{1.425} \\
& \textbf{W} & 28.3 & 36.2 & & \\ \midrule

\multirow{2}{*}{\textbf{Act. Recourse}} & \textbf{NW} & 46.5 & 54.9 & \multirow{2}{*}{1.101} & \multirow{2}{*}{\textbf{1.056}} \\
& \textbf{W} & 51.2 & 58.0 & & \\ \midrule

\multirow{2}{*}{\textbf{Random}} & \textbf{NW} & 4.9 & 28.9 & \multirow{2}{*}{1.571} & \multirow{2}{*}{1.076} \\
& \textbf{W} & 7.7 & 31.1 & & \\ \midrule

\multirow{2}{*}{\textbf{COLS}} & \textbf{NW} & 67.6 & 71.1 & \multirow{2}{*}{1.089} & \multirow{2}{*}{1.082} \\
& \textbf{W} & 73.6 & 76.9 & & \\ \midrule

\multirow{2}{*}{\textbf{P-COLS}} & \textbf{NW} & 72.5 & 74.6 & \multirow{2}{*}{\textbf{1.07}} & \multirow{2}{*}{1.092} \\
& \textbf{W} & 77.6 & 81.5 & & \\
\bottomrule
\end{tabular}
}
\captionof{table}{\label{tab:race_disparity}Fairness analysis of recourse methods for subgroups with respect to Race. 
\textbf{DIR}: Disparate Impact Ratio; \textbf{W}: White, \textbf{NW}: Non-White (Section \ref{exp:fairness}).}
\end{minipage}
\hfill
\begin{minipage}[b]{0.49\textwidth}
\centering
\includegraphics[width=\columnwidth]{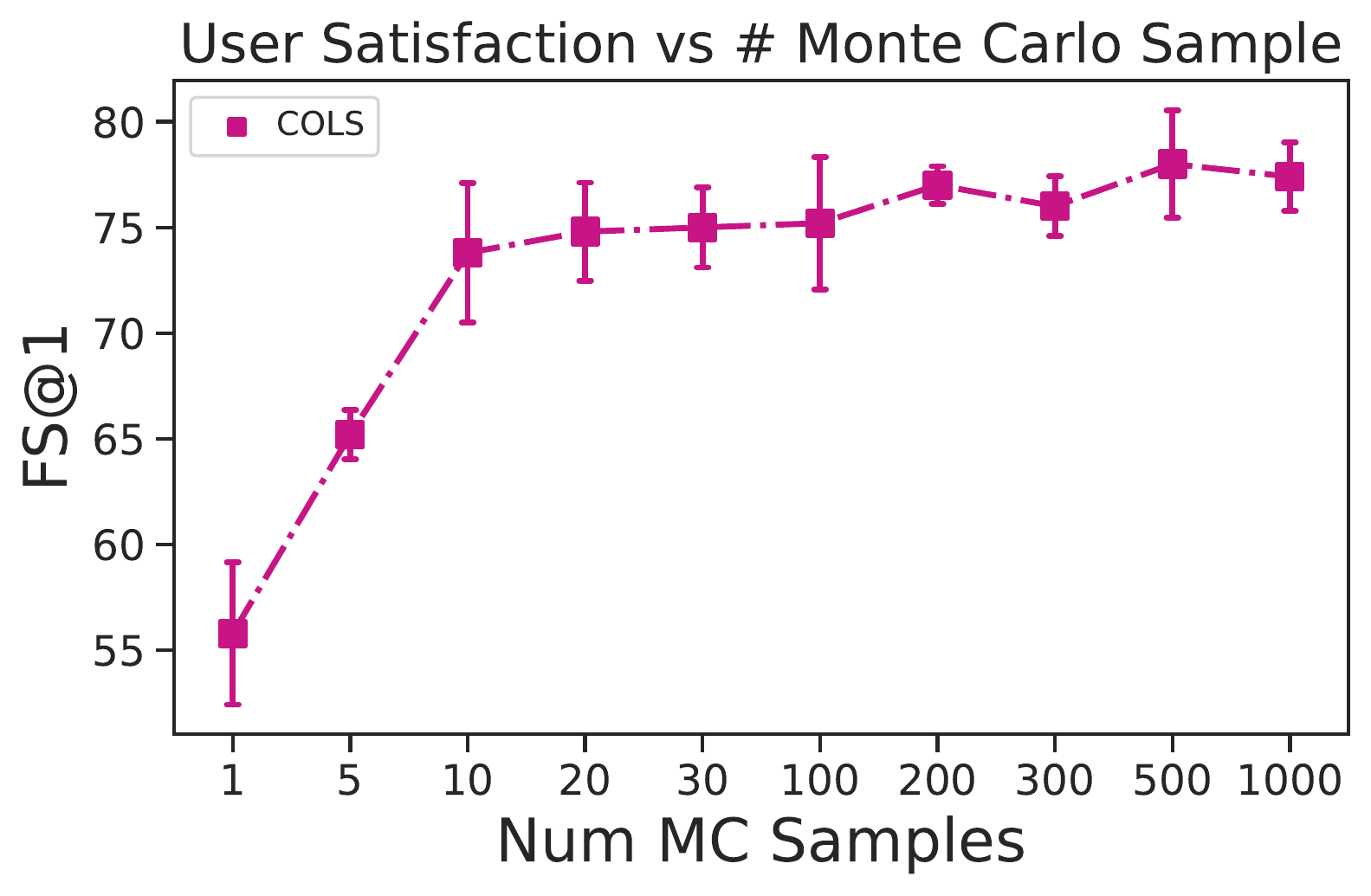}
\captionof{figure}{\label{fig:num_mcmc}Figure showing the performance of the COLS method as the number of Monte Carlo samples increase. These are the average number across 5 different runs along with standard deviation error bars. There is a steep increase and then the performance saturates. This implies that in practice we do not need a large number of samples to converge to the higher user satisfaction. Refer to Section \ref{exp:mc} for more details.}
\end{minipage}
\end{minipage}

\textbf{Q7. Does Method Performance Scale with Available Compute?}
\label{exp:budget}

\textbf{Design:} In this experiment on the Adult-Income dataset, we measure the change in performance of all the models as the number of access to the black-box model (budget) increases. Ideally, a good recourse method should be able to exploit these extra queries and use it to satisfy more users. We vary the allocated budget in the set $\{500, 1000, 2000, 3000, 5000, 10000\}$ and report the FS@1. We run the experiment on a random subset of $100$ users for $5$ independent runs and then report the average performance with standard deviation-based error bars in Figure \ref{fig:budget}. 

\textbf{Results:} In Figure \ref{fig:budget}, we can see that \textbf{as the allocated budget increases the performance of the COLS and P-COLS increases} and then saturates. This suggests that our method can exploit the additional black-box access to improve the performance. Other methods like AR and Face-Knn also show performance improvement but our method COLS and P-COLS consistently upper-bound their performance. \textbf{Our method satisfies approximately 70\% of the user with a small budget of 500} and quickly starts to saturate around a budget of $1000$. This suggests that \textbf{our methods are suitable even under tight budget constraints} as they can achieve good performance rapidly. For example, in a real-world scenario where the recourse method is deployed and has to cater to a large population, in such cases there might be budget constraints imposed onto the method where achieving good quality solution quickly is required. Lastly, for DICE and Random search the performance on the FS@1 increase by a very small margin and then stays constant as these methods are trying to optimize for different objectives which don't align well with user satisfaction as demonstrated in Section \ref{exp:main}.

\begin{table}[t]
\small
\centering
\resizebox{.85\textwidth}{!}{
\begin{tabular}{l|llc|c|c}
\toprule

 Feature Name & State Vector & Editable Features & Preference scores & Recourses & Cost \\ \midrule
 
 Age & 24 & No & 0 &  \multirow{4}{*}{$\begin{pmatrix} \text{Capital Loss: 0} \rightarrow \text{1} \end{pmatrix}$} & \multirow{4}{*}{0.009} \\
 Workclass & Private & No & 0 &  & \\
 Education-Num & 10 & No & 0 &  & \\
 Martial-Status & Married & No & 0 &  & \\
 Occupation & Other & Yes & 0.055 &\multirow{4}{*}{$\begin{pmatrix} \text{Occupation: Other} \rightarrow \text{Manager} \end{pmatrix}$} & \multirow{4}{*}{0.378} \\
 Relationship & Husband & No & 0 &  & \\
 Race & White & No & 0 &  & \\
 Gender & Male & No & 0 &  & \\
 Capital Gain & 0 & No & 0 & \multirow{4}{*}{$\begin{pmatrix} \text{Occupation: Other} \rightarrow \text{Manager}\\ \text{Capital Loss: 0} \rightarrow \text{1} \end{pmatrix}$} & \multirow{4}{*}{0.387} \\
 Capital Loss & 0 & Yes & 0.944 & & \\
 \# Work Hours & 40 & No & 0 &  & \\
 Country & US & No & 0 &  & \\ \midrule
 
  Age & 45 & No & 0 & \multirow{3}{*}{$\begin{pmatrix} \text{Capital Loss: 0} \rightarrow \text{1} \end{pmatrix}$} & \multirow{4}{*}{0.071} \\
 Workclass & Private & No & 0 &  & \\
 Education-Num & 7 & Yes & 0.537 &  & \\
 Martial-Status & Married & No & 0 & \multirow{3}{*}{$\begin{pmatrix} \text{Capital Gain: 0} \rightarrow \text{1} \end{pmatrix}$} & \multirow{4}{*}{0.106} \\
 Occupation & Other & No & 0 & & \\
 Relationship & Non-Husband & No & 0 &  & \\
 Race & White & No & 0 & \multirow{3}{*}{$\begin{pmatrix} \text{Education-Num: 7} \rightarrow \text{10}\end{pmatrix}$} & \multirow{4}{*}{0.187} \\
 Gender & Female & No & 0 &  & \\
 Capital Gain & 0 & Yes & 0.078 &  \\
 Capital Loss & 0 & Yes & 0.240 & \multirow{3}{*}{$\begin{pmatrix} \text{\# Work Hours: 32} \rightarrow \text{70} \end{pmatrix}$} & \multirow{4}{*}{0.695} \\
 \# Work Hours & 32 & Yes & 0.142 &  & \\
 Country & US & No & 0 &  & \\

\bottomrule 
\end{tabular}
}
\caption{\label{tab:examples}Table providing qualitative examples for two users from the dataset. We show each users state vector, the features that user is willing to edit, the preference scores for those editable features, the recourses provided and the cost of the generated recourses. In the first example we see that user highly prefers the feature \textit{capital loss} and the recourse which suggests edit to that has the lowest cost for the user. Whereas, the recourse which makes changes to both \textit{Occupation} and \textit{Capital Loss} has the highest cost as its changing multiple features. For the second user, we see that the most preferred feature is \textit{Education-Num} but the changes suggested in the recourse requires three steps 7-8-9-10, hence the cost for that recourse is not the lowest but still relatively low. Whereas, the recourse suggesting smaller changes to \textit{Capital Loss} which is the second most preferred feature has the lowest cost for the user.}
\end{table}

\textbf{Q8. Does providing more options to users help?}
\label{exp:cfs}

\textbf{Design:}
In this experiment, we measure the effect of having flexibility to provide the user with more options, i.e. a bigger set $\mathcal{S}$. The question here is that can the methods effectively exploit this advantage and provide lower cost solution sets to the user such that the overall user satisfaction is improved. In this experiment on the Adult-Income dataset, we take a random subset of $100$ users and fix the budget to $5000$, Monte Carlo cost sample is set to $1000$ and then vary the size of the set $\mathcal{S}$ in the set $\{1,2,3,5,10,20,30\}$. We restrict the size of the set to a maximum of $30$ as beyond a point it becomes hard for users to evaluate all the recourse options and decide which one to act upon. We run $5$ independent runs for all the data points and plot the mean performance along with standard deviation error bars. In Figure \ref{fig:num_cfs}, we plot the fraction of users satisfied @1 as the size of set $\mathcal{S}$ is increased. 

\textbf{Result:} We observed that \textbf{COLS and P-COLS monotonically increase the FS@1 metric as $|\mathcal{S}|$ increases} from $1$ to $30$. This is consistent with the intuition behind our methods (See Figure \ref{fig:diagram}, section \ref{sec:approx_objective}, \ref{sec:supp_design} for more details). It is a fundamental property of our objective that as $|\mathcal{S}|$ increases towards $M$ which is $1000$ in this case, then the quality of the solution set should increase and reach the best possible value that can be provided under the user's cost function. We note empirically that \textbf{smaller set size $|\mathcal{S}|$ between 3 to 10 is enough in most practical cases} to reach close to maximum performance. Additionally, even with $|\mathcal{S}| \in \{1,2,3\}$ our methods significantly outperform all the other methods in terms of the number of users satisfied. This property is useful in real-world scenarios where the deployed recourse method can provide as little as $3$ options while still satisfying a large fraction of users. Additionally, we also see improvement in the case of AR and Face-Knn methods as $|\mathcal{S}|$ increases. Note that Randoms Search's performance doesn't change as we increase the set size because the method doesn't take local steps from the best set and samples random points from a very large space, hence it is much harder to end up with low-cost counterfactuals.

\textbf{Q9. Does increase the number of Monte Carlo samples help with user satisfaction?}
\label{exp:mc}

\textbf{Design:} In this experiment, we want to demonstrate the effect of increasing the number of Monte Carlo samples on the performance of our COLS method. We take a random subset of $100$ users, a budget of $5000$, $|\mathcal{S}| = 10$. We vary the number of Monte Carlo samples (M) in the set $\{1, 5, 10, 20, 30, 100, 200, 300, 500, 1000\}$ and compute the user satisfaction. We ran 5 different runs with different Monte Carlo samples and show the average FS@1 along with the standard deviation in the Figure \ref{fig:num_mcmc}.\\
\textbf{Results:} We observe that \textbf{as the number of Monte Carlo samples increases, the performance of the method on the FS@1 metric monotonically increases.} This supports the intuition underlying our method (see Figure \ref{fig:diagram}). That is, given a user with a cost function $\mathcal{C}_u$ as we get more and more samples from the cost distribution $\mathcal{D}_u$ the probability of having a cost sample similar to $\mathcal{C}_u$ increases and hence the fraction of satisfied users increase. It is important to note that \textbf{empirically the method's performance approaches maximum user satisfaction with as low as 20 Monte Carlo samples.} In real-world scenarios, where the deployed model is catering to a large population this can lead to small recourse generation time, hence making it more practical.

\begin{table}[t]
\footnotesize
\centering
\resizebox{0.8\textwidth}{!}{
\begin{tabular}{ccccccccc}
\toprule
\textbf{Data}                       & \textbf{Method} & \multicolumn{7}{c}{\textbf{Metrics}}                                                                                    \\ \midrule
                                 &       & \multicolumn{3}{c}{\textbf{Cost Metrics}}  & \multicolumn{3}{c}{\textbf{Distance Metrics}}  \\ \cmidrule(lr){3-5} \cmidrule(lr){6-8}
                                 
                                 &       & \textbf{FS@1} & \textbf{PAC$(\downarrow)$} & \textbf{Cov} & \textbf{Div} & \textbf{Prox} & \textbf{Spars} & \textbf{Val} \\ \cmidrule(lr){3-9}

\multirow{4}{*}{\textbf{Adult-Income - NB}} &  \textbf{DICE}   & 6.28 & 1.45  & 27.01  & 53.01  & 57.02  & 47.80  & 86.20 \\
 &  \textbf{Random}   & 0.08   & 2.42  & 17.41  & \textbf{70.35}  & 33.32  & 22.45  & 75.71 \\ \cmidrule{2-9}
&  \textbf{COLS}    & 72.67    & \textbf{0.36}  & \textbf{74.60}  & 29.27  & \textbf{79.06}  & \textbf{76.64}  & \textbf{97.85} \\
&  \textbf{P-COLS}  & \textbf{70.03}    & 0.39  & 72.81  & 29.85  & 78.45  & 76.29  & 92.30 \\ \midrule
\multirow{4}{*}{\textbf{COMPAS - NB}} &  \textbf{DICE}   & 14.86    & 1.02  & 25.45  & 27.88  & 82.38  & 69.44  & \textbf{99.86} \\
&  \textbf{Random}   & 1.31    & 1.87  & 21.76  & \textbf{49.07}  & 54.10  & 42.34  & 67.82 \\ \cmidrule{2-9}
 &  \textbf{COLS}    & 67.34    & \textbf{0.31}  & 68.11  & 20.53  & 85.47  & 82.34  & 95.97 \\
 &  \textbf{P-COLS}  & \textbf{70.86}    & 0.35  & \textbf{72.03}  & 21.03  & \textbf{85.48}  & \textbf{82.88}  & 91.93 \\

\bottomrule
\end{tabular}
}
\caption{\label{tab:supp_main}Table comparing different recourse methods across various cost and distance metrics on Non-Binary versions of the datasets (Section \ref{sec:supp_datasets}).The numbers reported are averaged across $5$ different runs. Variance values have been as 89\% of them were lower than 0.05, with the maximum being 0.86. FS@1: Fraction of users satisfied at $k=1$. PAC: Population Average Cost. Cov: Population Coverage. For all the metrics higher is better except for PAC where lower is better.}
\end{table}

\textbf{Q10. Qualitative examples of the recourses generated for some of the users.}

In Table \ref{tab:examples}, we show a few examples of users along with their state vector, their editable features, their preference scores along with the recourses provided to them and their cost. 

\textbf{Q11. Comparison of methods on Non Binary Dataset?}
\label{exp:supp_main}

In Table \ref{tab:supp_main}, we show the results on the non-binary version of the dataset. We observe similar performance on and trends in these results as well. COLS and P-COLS performs the best in terms of user satisfaction. 

\textbf{Q12. Robustness to black-box model architecture families and randomness?}
\label{exp:supp_seed}

In this experiment we demonstrate the result of our model when we train the same ANN architecture with different random seed (Table \ref{tab:supp_main_model2}) and when we change the model family to a logistic regression classifier (Table \ref{tab:supp_main_logistic}). These obtained results have similar trends and demonstrate the effectiveness and robustness of our methods COLS and P-COLS which consistently satisfy cover and satisfy more users with low average population costs.
In Table \ref{tab:supp_main_model2}, we show the results when we train another black-box model with a different seed to see the effect of having a different trained model from the same model family. 

\begin{table}
\footnotesize
\centering
\resizebox{.9\textwidth}{!}{
\begin{tabular}{ccccccccc}
\toprule
\textbf{Data}                       & \textbf{Method} & \multicolumn{7}{c}{\textbf{Metrics}}                                                                                    \\ \midrule
                                 &       & \multicolumn{3}{c}{\textbf{Cost Metrics}}  & \multicolumn{3}{c}{\textbf{Distance Metrics}}  \\ \cmidrule(lr){3-5} \cmidrule(lr){6-8}
                                 
                                 &       & \textbf{FS@1} & \textbf{PAC$(\downarrow)$} & \textbf{Cov} & \textbf{Div} & \textbf{Prox} & \textbf{Spars} & \textbf{Val} \\ \cmidrule(lr){3-9}

\multirow{7}{*}{\textbf{Adult-Income}}  & \textbf{DICE} & 2.70 &	1.24 &	7.10 &	3.80 &	66.20 &	47.30 &	97.80 \\
& \textbf{Face-Eps} & 13.32 &	0.79 &	19.88 &	5.43 &	91.97 &	74.80 &	\textbf{100.00} \\
& \textbf{Face-Knn} & 21.78 &	0.83 &	34.13 &	8.67 &	88.68 &	71.43 &	\textbf{100.00} \\
& \textbf{Act. Recourse} & 46.55 &	0.58 &	53.82 &	19.07 &	74.33 &	73.25 &	80.72 \\
& \textbf{Random} & 5.71 &	1.42 &	28.24 &	\textbf{48.93} &	55.10 &	39.30 &	78.73 \\ \cmidrule{2-9}
& \textbf{COLS} & 75.12 &	\textbf{0.36} &	77.40 &	25.43 &	81.00 &	77.70 &	98.28 \\
& \textbf{P-COLS} & \textbf{75.76} &	0.38 &	\textbf{79.14} &	25.54 &	\textbf{81.84} &	\textbf{78.38} &	95.10 \\ \midrule
\multirow{7}{*}{\textbf{COMPAS}} & \textbf{DICE} & 0.90 &	0.88 &	1.50 &	12.50 &	63.90 &	30.70 &	99.30 \\
& \textbf{Face-Eps} & 6.80 &	0.29 &	6.80 &	2.40 &	\textbf{95.00} &	60.40 &	\textbf{100.00} \\
& \textbf{Face-Knn} & 6.80 &	0.29 &	6.80 &	2.40 &	94.90 &	60.30 &	\textbf{100.00} \\
& \textbf{Act. Recourse} & 56.24 &	0.45 &	58.48 &	9.72 &	80.12 &	\textbf{73.62} &	39.10 \\
& \textbf{Random} & 27.44 &	0.78 &	35.70 &	\textbf{41.76} &	58.14 &	33.06 &	49.34 \\ \cmidrule{2-9}
& \textbf{COLS} & 77.08 &	\textbf{0.24} &	77.90 &	29.33 &	76.90 &	68.87 &	95.78 \\
& \textbf{P-COLS} & \textbf{78.32} &	\textbf{0.24} &	\textbf{79.02} &	29.02 &	77.88 &	70.08 &	92.10 \\

\bottomrule
\end{tabular}
}
\caption{\label{tab:supp_main_model2}Table comparing different recourse methods across various cost and distance metrics for a black-box model with different seed but belonging to the same model family. The numbers reported are averaged across $5$ different runs.
}
\end{table}

\begin{table}
\footnotesize
\centering
\resizebox{.9\textwidth}{!}{
\begin{tabular}{ccccccccc}
\toprule
\textbf{Data}                       & \textbf{Method} & \multicolumn{7}{c}{\textbf{Metrics}}                                                                                    \\ \midrule
                                 &       & \multicolumn{3}{c}{\textbf{Cost Metrics}}  & \multicolumn{3}{c}{\textbf{Distance Metrics}}  \\ \cmidrule(lr){3-5} \cmidrule(lr){6-8}
                                 
                                 &       & \textbf{FS@1} & \textbf{PAC$(\downarrow)$} & \textbf{Cov} & \textbf{Div} & \textbf{Prox} & \textbf{Spars} & \textbf{Val} \\ \cmidrule(lr){3-9}

\multirow{7}{*}{\textbf{Adult-Income}} & \textbf{D}	& 1.30	& 1.47	& 7.10	& 6.50	& 64.80	& 49.20	& 76.60 \\
& \textbf{FE}	& 2.82	& 0.99	& 5.46	& 9.44	& 83.08	& 65.90	& \textbf{100.00} \\
& \textbf{FK}	& 17.04	& 0.90	& 28.08	& 7.22	& 83.98	& 67.80	& \textbf{100.00} \\
& \textbf{AR}	& 44.40	& 0.62	& 52.62	& 21.74	& 74.00	& 72.92	& 87.58 \\
& \textbf{R}	& 3.60	& 1.59	& 24.10	& \textbf{48.14}	& 54.76	& 38.94	& 79.46 \\ \cmidrule{2-9}
& \textbf{COLS}	& 67.93	& \textbf{0.39}	& 69.97	& 27.83	& 78.40	& 74.43	& 99.13 \\
& \textbf{P-COLS}	& \textbf{69.17}	& 0.40	& \textbf{71.57}	& 27.20	& \textbf{79.30}	& \textbf{76.70}	& 95.67 \\ \midrule
\multirow{7}{*}{\textbf{COMPAS}} & \textbf{D} &	0.00 &	-	& 0.00	& 11.10	& 63.10	& 29.20	& 100.00 \\
& \textbf{FE}	& 6.30	& \textbf{0.16}	& 6.30	& 3.60	& \textbf{95.00}	& 60.50	& \textbf{100.00} \\
& \textbf{FK}	& 6.30	& \textbf{0.16}	& 6.30	& 3.60	& \textbf{95.00}	& 60.50	& \textbf{100.00} \\
& \textbf{AR}	& 74.32	& 0.31	& 74.32	& 15.66	& 80.98	& 74.26	& 53.66 \\
& \textbf{R}	& 28.76	& 0.77	& 36.96	& \textbf{43.22}	& 56.22	& 32.00	& 82.10 \\ \cmidrule{2-9}
& \textbf{COLS}	& 87.88	& 0.18	& 87.88	& 31.92	& 76.93	& 71.63	& 89.33 \\
& \textbf{P-COLS}	& \textbf{89.25}	& 0.17	& \textbf{89.25}	& 28.17	& 81.13	& \textbf{74.73}	& 91.62 \\

\bottomrule
\end{tabular}
}
\caption{\label{tab:supp_main_logistic}Table comparing different recourse methods across various cost and distance metrics for a logistic regression black-box model. The numbers reported are averaged across $5$ different runs.
}
\end{table}

\begin{algorithm}[tbh!]
\caption{Sampling procedure for Percentile Cost Mean \label{alg:percentile}}
\SetKwBlock{Begin}{function}{end function}
\Begin($\text{PerCost} {(} \vs, p^{(f_i)}, f_i, \mathcal{F}_p {)}$)
{
    \tcp{ \footnotesize $s_i$ value of feature $f_i$ in s.}
    \uIf{$f_i \notin \mathcal{F}_p$}{
        $\mu^{(f_i)}(s_i, .) = \infty $\\
        $\mu^{(f_i)}(s_i, s_i) = 0 $}
    \uElse{
        \uIf{$f_i$ is ordered}{
            \uIf {$f_i$ can only increase}{
                $\mu^{(f_i)}(s_i, x) = $
                        $\begin{cases} 
                        |getPercentile(x) - getPercentile(s_i)| & \forall x > s_i\\
                        0 & \forall x = s_i\\
                        \infty & \forall x < s_i
                       \end{cases}$}
            \uElseIf {$f_i$ can only decrease}{
                $\mu^{(f_i)}(s_i, x) = $
                        $\begin{cases} 
                        |getPercentile(s_i) - getPercentile(x)| & \forall x < s_i\\
                        0 & \forall x = s_i\\
                        \infty & \forall x > s_i
                       \end{cases}$}
            \uElseIf {$f_i$ can both increase or decrease}{
                $\mu^{(f_i)}(s_i, x) = $
                        $\begin{cases} 
                        |getPercentile(x) - getPercentile(s_i)| & \forall x > s_i\\
                        0 & \forall x = s_i\\
                        |getPercentile(s_i) - getPercentile(x)| & \forall x < s_i
                       \end{cases}$}
        }
        \uElseIf{$f_i$ is unordered}{
            $\mu^{(f_i)}(s_i, .) = Uniform(0,1)$
        }
    }
    $\mu^{(f_i)}(s_i, . ) \gets \mu^{(f_i)}(s_i, . ) * (1-p^{(f_i)})$\\
    $\sigma^{(f_i)}(s_i, . ) \gets 0.01$\\
    \Return{$\mu^{(f_i)}$, $\sigma^{(f_i)}$}
}
\end{algorithm}

\begin{algorithm}[tbh!]
\caption{Sampling procedure for Linear Cost Means \label{alg:linear}}
\SetKwBlock{Begin}{function}{end function}
\Begin($\text{LinCost} {(} \vs, p^{(f_i)}, f_i, \mathcal{F}_p {)}$)
{
    \uIf{$f_i \notin \mathcal{F}_p$}{
        $\mu^{(f_i)}(s_i, .) = \infty $\\
        $\mu^{(f_i)}(s_i, s_i) = 0 $}
    \uElse{
        \uIf{$f_i$ is ordered}{
            \uIf {$f_i$ can only increase}{
                $\mu^{(f_i)}(s_i, x) = $
                        $\begin{cases} 
                        \frac{|\{y~|~ y>s_i \land y \leq x\}|}{|\{y ~|~ y > s_i\}|} & \forall x > s_i \\
                        0 & \forall x = s_i \\
                        \infty & \forall x < s_i \\
                      \end{cases}$
                       }
            \uElseIf {$f_i$ can only decrease}{
                $\mu^{(f_i)}(s_i, x) = $
                        $\begin{cases} 
                        \frac{|\{y~|~ y < s_i \land y \geq x\}|}{|\{y ~|~ y < s_i\}|} & \forall x < s_i \\
                        0 & \forall x = s_i \\
                        \infty & \forall x > s_i \\
                      \end{cases}$
                       }
            \uElseIf {$f_i$ can both increase or decrease}{
                $\mu^{(f_i)}(s_i, x) = $
                        $\begin{cases} 
                        \frac{|\{y~|~ y > s_i \land y \leq x\}|}{|\{y ~|~ y > s_i\}|} & \forall x > s_i \\ 
                        0 & \forall x = s_i \\
                        \frac{|\{y~|~ y < s_i \land y \geq x\}|}{|\{y ~|~ y < s_i\}|} & \forall x < s_i \\
                      \end{cases}$
                       }
        }
        \uElseIf{$f_i$ is unordered}{
            $\mu^{(f_i)}(s_i, .) = Uniform(0,1)$
        }
    }
    $\mu^{(f_i)}(s_i, . ) \gets \mu^{(f_i)}(s_i, . ) * (1-p^{(f_i)})$\\
    $\sigma^{(f_i)}(s_i, . ) \gets 0.01$\\
    \Return{$\mu^{(f_i)}$, $\sigma^{(f_i)}$}
}
\end{algorithm}


\section{Appendix - Objective and Optimization}

\subsection{Proposed Method}
\label{sec:supp_method}
\subsubsection{Other Objectives}
\label{sec:supp_other_objective}
To obtain feasible a counterfactual set, past works have used various objective terms. We list objectives below from methods we compare with. 

\textbf{1. DICE} \citep{dice_fat} optimizes for a combination of Distance Metrics like \textit{diversity} and \textit{proximity}. They model diversity via Determinantal Point Processes \citep{dpp} adopted for solving subset selection problems with diversity constraints. They use determinant of the kernel matrix given by the counterfactuals as their diversity objective as defined below.
\[dpp\_diversity(\mathcal{S}) = det(\textbf{K}), ~\text{where} \textbf{K}_{ij} = \frac{1}{1 + dist(\vs_i, \vs_j)}\]

Here, $dist(\vs_i, \vs_j)$ is the normalized distance metric as defined in \cite{Wachter2017CounterfactualEW} between two state vectors. \textit{Proximity} is defined in terms of the distance between the original state vector and the counterfacutals, $prox(\vx, \mathcal{S}) = 1 - \frac{1}{N}\sum_{i=1}^{|\mathcal{S}|} dist(\vx, \mathcal{S}_i)$, where $\mathcal{S}_i$ is a counterfactual.

\textbf{2. Actionable Recourse} \citep{ar_fat} work under the assumption that all features have equal preference scores for all the users. They define  cost function based on the log-percentile shift is given by, 
\[cost(\vs + a; \vs) = \sum_{j \in \mathcal{J}_{A}} \log \frac{1 - Q_j(\vs_j + a_j)}{1 - Q_j(\vs_j)}\]
where $Q_j(.)$ is the cumulative distribution function of $\vs_j$ in the target population, $\mathcal{J}_{A}$ is the set of actionable features and $a_j$ is the action performed on the feature $j$. 

\begin{algorithm}
\caption{Cost-Optimized Local Search Algorithm \label{alg:cols}}
\DontPrintSemicolon
\KwIn{A state vector $\vs$, $\{\mathcal{C}_i\}_{i=1}^{M} \sim \mathcal{D}_u$ cost distributions}
\KwOut{$\mathcal{S}^{best}$, a set of generated counterfactuals of size $N$.}
\SetKwBlock{Begin}{function}{end function}
\Begin($\text{LocalSearch} {(} \vs, \{\mathcal{C}_i\}_{i=1}^{M}, \text{hammingDistance} = 2 ${)})
{   
    \Init{}{
    $\mathcal{S}^{best} \in \mathbb{R}^{N\times d} \gets \text{pertubCFS}(\vs, \text{hammingDistance})$\Comment*[r]{ \footnotesize Perturb $\vs$, $N$ times.}
    $ \textbf{C}^b \gets \text{getCostMatrix}(\vs, \mathcal{S}^{best} ; \{\mathcal{C}_i\}_{i=1}^{M}) $ \Comment*[r]{ \footnotesize Incurred costs for $\mathcal{S}^{best}$.}
    }
    \While{$\text{usedBudget} < \text{Budget}$}{
        $\mathcal{S} \in \mathbb{R}^{N\times d} \gets \text{pertubCFS}(\mathcal{S}^{best}, \text{hammingDistance})$
        
        $\textbf{C} \in \mathbb{R}^{N\times M} \gets \text{getCostMatrix}(\vs, \mathcal{S} ; \{\mathcal{C}_i\}_{i=1}^{M})$ \Comment*[r]{ \footnotesize Incurred costs for the $\mathcal{S}$.}
        
        \tcp{ \footnotesize  $\textbf{B}_{ij}$ = Change in objective when $\mathcal{S}^{best}[i] \gets \mathcal{S}[j]$.}
        
        $\textbf{B} \in \mathbb{R}^{N \times N} \gets \text{computeBenefits}(\textbf{C}^b, \textbf{C})$ \Comment*[r]{ \footnotesize Refer to Algorithm \ref{alg:theorem}}
        
        \tcp{ \footnotesize Greedily select which pairs to swap given \textbf{B}}
        $\text{replaceIndices} \gets \text{getReplaceIdx}(\textbf{B})$
        
        \tcp{ \footnotesize Swap these pairs and update $\textbf{C}^b$.}
        \ForAll{$\text{originalIdx}, \text{replaceIdx} \in \text{replaceIndices}$}{
            $\mathcal{S}^{best}[\text{originalIdx}] = \mathcal{S}[\text{replaceIdx}]$
        }
        $\textbf{C}^b \gets \text{getCostMatrix}(\vs, \mathcal{S}^{best} ; \{\mathcal{C}_i\}_{i=1}^{M})$
    }
    \Return{$\mathcal{S}^{best}, \textbf{C}^b$}\; 
}

\end{algorithm}

\begin{algorithm}
\caption{Algorithm for Theorem \ref{theorem} \label{alg:theorem}}
\DontPrintSemicolon
\KwIn{$\textbf{C}^b, \textbf{C} \in \mathbb{R}^{N \times M}$ matrices containing the costs with respect to all cost samples..}
\KwOut{$\textbf{B} \in \mathbb{R}^{N \times N}$, matrix containing the benefits of replacing pairs from $\mathcal{S}^{best}_{t-1} \times \mathcal{S}_{t}$}
\SetKwBlock{Begin}{function}{end function}
\Begin($\text{computeBenefits} {(} \textbf{C}^b, \textbf{C} {)}$)
{   
    \Init{}{
    $\textbf{B} \in \mathbb{R}^{N \times N} \gets \textbf{0}$\\
    }
    \tcp{ \footnotesize Find the indices of the best and second best counterfactual in $\mathcal{S}^{best}$ for each of the M cost function.}
    $\vb^1 \in \mathbb{R}^{M} = \argmax_{i} \textbf{C}^b_{ij}$\\
    $\vb^2 \in \mathbb{R}^{M} = \text{arg second max}_{i} \textbf{C}^b_{ij}$\\
    \tcp{ \footnotesize Iterate over all pairs of counterfactuals.}
    \ForAll{$p, q \in [N] \times [N]$} 
    {
        \tcp{ \footnotesize Iterate over cost functions for which $p^{th}$ counterfactual in $\mathcal{S}^{best}$ has the minimum cost. }
        \ForAll{$r \in \{i \in [M] ~|~ \vb^1_i = p\}$}
        {
            
            \uIf{$\textbf{C}^b_{pr} > \textbf{C}_{qr} $}
            {
                \tcp{ \footnotesize  This replacement reduces the cost of $\mathcal{S}^{best}$ by $\textbf{C}^b_{pr} - \textbf{C}_{qr}$.}
                $\textbf{B}_{pq} += \textbf{C}^b_{pr} - \textbf{C}_{qr} $
            }
            \uElse
            {
                \tcp{ \footnotesize $\textbf{C}^b_{\vb^2_r, r} =$ cost of second best counterfactual in $\mathcal{S}^{best}$ for $r^{th}$ cost function.}
                $\textbf{B}_{pq} += \textbf{C}^b_{pr} - min(\textbf{C}_{qr},  \textbf{C}^b_{\vb^2_r, r})$
            }
        }
    } 
    \Return{$\textbf{B}$}\; 
}

\end{algorithm}

\subsection{Optimization Methods}

\textbf{Notation:}
We assume that we have a dataset with features $\mathcal{F} = \{f_1, f_2, ... f_k\}$. Each feature can either be continuous $\mathcal{F}^{con} \subset \mathcal{F}$ or categorical $\mathcal{F}^{cat} \subset \mathcal{F}$. Each continuous feature $f_i^{con}$ takes values in the range $[r_i^{min}, r_i^{max}]$, which we discretize to integer values. 
For a continuous feature $f_i$, we define the range $Q^{(f_i)} = \{k \in \mathbb{Z} : k \in [r_i^{min}, r_i^{max}]\}$ and for a categorical feature $f_i$, we define it as $Q^{(f_i)} = \{q^{f_i}_1, q^{f_i}_2, ...,  q^{f_i}_{d_i}\}$, where $q^{f_i}_{(.)}$ are the states that feature $f_i$ can take.
Features can either be mutable ($\mathcal{F}^{m}$), conditionally mutable ($\mathcal{F}^{cm}$), or immutable ($\mathcal{F}^{\oslash}$), according to the real-world causal processes that generate the data. 
Mutable features can transition from between any pair of states in $Q^{(f_i)}$; conditionally mutable features can transition between pairs of states only when permitted by certain conditions; and immutable features cannot be changed under any circumstances. 
For example, \textit{Race} is an immutable feature \citep{dice_fat}, \textit{Age} and \textit{Education} are conditionally mutable (cannot be decreased under any circumstances), and \textit{number of work hours} is mutable (can both increase and decrease).
Lastly, while continuous features inherently define an ordering in its values, categorical features can either be ordered or unordered based on its semantic meaning. For instance, \textit{Age} is an ordered feature that is conditionally mutable (can only increase).

\subsubsection{Hierarchical Cost Sampling Procedure}
\label{sec:supp_sampling}

To optimize for EMC, we need a plausible distribution which can model users' cost functions. We propose a hierarchical cost sampling distribution which provides cost samples that are a linear combination of \textit{percentile shift cost} \citep{ar_fat} and \textit{linear cost}, where the weights of this combination are user-specific. 
\textit{Percentile shift cost} for ordered features is proportional to the change in a feature's percentile associated with the change from an old feature value to a new one. 
E.g., if a user is asked to increase the number of work hours from $40$ to $70$, then given the whole dataset, we can estimate the percentile of users working $40$ and $70$ hours a week. The cost incurred is then proportional to the difference in these percentiles. 
The \textit{Linear cost} for ordered features is proportional to the number of intermediate states a user will have to go through while transitioning from their current state to the final state. 
E.g., if a user is asked to change their education level from \textit{High-school} to \textit{Masters} then there are two steps involved in the process. First, they need to get a \textit{Bachelors} degree and then a \textit{Masters} degree in which case, the user's cost is proportional to 2 because of the two steps involved in the process.
To sample a cost function $\mathcal{C} = \{ \mathcal{C}^{(f)}(i, j): \mathbb{R}^{|f| \times |f|} \rightarrow [0,1] \cup \{\infty\} ~ | ~ \forall f \in \mathcal{F} \}$,
we independently sample $\mathcal{C}^{(f)}$ for each feature (see Algorithm \ref{alg:sampling}).
We first randomly sample a subset of editable features for the user, and then we sample feature preference score $\vp_u$ from a Dirichlet distribution with a uniform prior over selected features. These will be used to scale costs, such that a higher value of $\vp_u^{(f)}$ implies a lower transition cost for feature $f$.
For both the percentile and linear cost, the cost $\mathcal{C}^{(f)}_{ij}$ of transitioning from feature state $i \rightarrow j$, is sampled from a Beta distribution on the interval $[0,1]$. 
The mean of this distribution depends on the types of cost (linear or percentile) and the feature type (ordered or unordered). 
Here, we first obtain one mean for linear cost ($\mu^{(f, lin)}_{ij}$) and one mean for percentile cost ($\mu^{(f, perc)}_{ij}$) and then combine them to form a single Beta mean. Each of the two means is proportional to the change in the feature (in either linear or percentile terms) when the feature is ordered. For unordered features, the mean is randomly sampled from the unit interval (see Algorithms \ref{alg:linear}, \ref{alg:percentile}).
Then, the linear and percentile means are multiplied with $(1-\vp^{(f)}_u)$ to scale the transition cost according to the feature preference score. Next, the two means are combined to obtain a single mean, $\mu^{(f)}_{ij} = \alpha * \mu^{(f, lin)}_{ij} + (1 - \alpha) * \mu^{(f, perc)}_{ij}$, where $\alpha \in [0,1]$ represents whether a user thinks of cost in terms of linear or percentile shift (sampled randomly from unit interval).
Note that the value $\mu_{ij}^{(f)}$ are monotonic, i.e. if the user has to make more drastic changes to the feature, then the associated cost will be higher.
The variance for the Beta distribution, $\sigma^{(f)}_{ij}$, is set to constant value of $0.01$. 
Finally, the cost $\mathcal{C}^{(f)}_{ij}$ is sampled from Beta$(\mu^{(f)}_{ij}, \sigma^{(f)}_{ij})$. 
We emphasize that this sampling procedure allows users to partially specify their cost functions, e.g. by denoting which features they prefer to edit (the Dirichlet mean) or the relative difficulty of editing one feature versus another ($\vp_u$). If we set $\alpha = 0$, then the resulting distribution is $\mathcal{D}_{perc}$ and with $\alpha = 1$ we get the distribution $\mathcal{D}_{lin}$.

\subsubsection{Merging Counterfactual Sets}

When searching for a good solution set, it would be useful to have the option of improving on the best set we have obtained so far using individual counterfactuals in the next candidate set we see, rather than waiting for a new, higher-scoring set to come along. While optimizing for objectives like diversity, which operate over all pairs of elements in the set, it is computationally complex to evaluate the change in the objective function if one element of the set is replaced by a new one. To evaluate the change in objective in such cases, we need iterate over all pairs of element in the best and the candidate set and then evaluate the objective for the whole set again. The iteration over both the sets here is not the hard part but the computation that needs to be done within. For our objective, we can compute costs for individual recourses rather than sets, meaning we can do a trivial operation to compute the benefits of each pair replacement. But, if we wanted to do this with diversity then for each pair of replacement we need to compute additional $\mathcal{S}$ distances for each replacement because the distance of the new replace vector needs to be computed with respect to all the other vectors, for each iteration of the nested loop. This quickly makes it infeasible to improve the best set by replacing individual candidates with the best set elements.
However, for metrics where it is easy to evaluate the effect of individual elements on the objective function, we can easily merge the best set and any other set $\mathcal{S}_t$ from time $t$ to monotonically increase the objective function value.

In our objective function, EMC, we can compute the goodness of individual counterfactuals with respect to all the Monte Carlo samples \citep{monte_carlo}. Given a set of counterfactuals we can obtain a matrix of incurred cost $\textbf{C} \in \mathbb{R}^{N \times M}$, which specifies the cost of each counterfactual for each of the Monte Carlo samples. We can use this to update the best set $\mathcal{S}^{best}$ using elements from the perturbed set $\mathcal{S}_t$ at time $t$. This procedure is defined in algorithm \ref{alg:theorem}. It iterates over all pairs of element in $\vs_i \in \mathcal{S}^{best}$ and $\vs_j \in \mathcal{S}_t$ and computes the change that will occur in the objective function by replacing $\vs_i \rightarrow \vs_j$. Note that we are not recomputing the costs here. Given $\mathcal{S}^{best}$, $\mathcal{S}_t$, $\textbf{C}^b$ and $\textbf{C}$, we can guarantee that we will update the best set $\mathcal{S}_{best}$ in a way to improve the mean of the minimum costs incurred for all the Monte Carlo samples. This is shown in algorithm \ref{alg:theorem} and the monotonicity of the EMC objective under this case can be formally stated as,

\begin{theorem}[Monotonicity of Cost-Optimized Local Search Algorithm]
\label{thm:supp_monotonicity}
Given the best set, $\mathcal{S}^{best}_{t-1} \in \mathbb{R}^{N \times d}$, the candidate counterfactual at iteration $t$, $\mathcal{S}_t \in \mathbb{R}^{N \times d}$, the matrix $\textbf{C}^b \in \mathbb{R}^{N \times M}$ and $\textbf{C} \in \mathbb{R}^{N \times M}$ containing the incurred cost of each counterfactual in $\mathcal{S}^{best}_{t-1}$ and $\mathcal{S}_{t}$ with respect to all the $M$ sampled cost functions $\{\mathcal{C}_i\}_{i=1}^{M}$, there always exist a  $\mathcal{S}^{best}_{t}$ constructed from $\mathcal{S}^{best}_{t-1}$ and $\mathcal{S}_t$ such that 
\begin{equation} \footnotesize
    \text{EMC}(\vs_u, \mathcal{S}^{best}_{t} ; \{\mathcal{C}_i\}_{i=1}^M) \leq \text{EMC}(\vs_u, \mathcal{S}^{best}_{t-1} ; \{\mathcal{C}_i\}_{i=1}^M)   \nonumber
\end{equation}
\end{theorem}
\begin{proof}
\label{proof}
To prove this theorem, we construct a procedure that ensures that the EMC is monotonic. For this procedure, we prove that the monotonicity of EMC holds. Check algorithm \ref{alg:theorem} for a constructive procedure for this proof, which is more intuitive to understand. 

We start off by noting that each element of $\textbf{C}^b_{ij}$ is the cost of the $i^{th}$ counterfactual $\vs^b_i$ in the best set $\mathcal{S}^{best}_{t-1}$ with respect to the cost function $\mathcal{C}_j$ given by $\text{Cost}(\vs_u, \vs^b_i; \mathcal{C}_j)$. Similarly $\textbf{C}_{ij} = \text{Cost}(\vs_u, \vs_i; \mathcal{C}_j)$ where $\vs_i$ is the $i^{th}$ candidate counterfactual. Note that, the EMC is the average of the MinCost with respect to all the sampled cost function $\mathcal{C}_j$. What this means is that given a pair of counterfactual from $\mathcal{S}^{best}_{t-1} \times \mathcal{S}_{t}$ and for each $\mathcal{C}_j$, we can compute the change in the MinCost which we describe later.
These replacements can lead to an increase in the cost with respect to certain cost function but the overall reduction depend on the aggregate change over all the cost functions. Given this, for each replacement candidate pair in $\mathcal{S}^{best}_{t-1} \times \mathcal{S}_{t}$, we can compute the change in EMC by summing up the changes in the MinCost across all cost functions $\mathcal{C}_j$; this is called the cost-benefit for this replacement pair. The cost benefit can be negative for certain replacements as well if the candidate counterfactual increases the cost across all the cost functions. The pairs with the highest positive cost benefits are replaced to construct the set $\mathcal{S}^{best}_t$, if no pair has a positive benefit then we keep set $\mathcal{S}^{best}_{t-1} = \mathcal{S}^{best}_{t}$. Hence, this procedure monotonically reduces EMC. We now specify how the change in MinCost can be computed to complete the proof.

To compute the change in MinCost for a single cost function $\mathcal{C}_i$, first we find the counterfactual in $\mathcal{S}^{best}_{t-1}$ with the lowest and second lowest cost which we denote by $\vs^b_{l_1}$ and $\vs^b_{l_2}$. These are the counterfactuals which can affect the MinCost with respect to a particular cost function $\mathcal{C}_i$. This is true because when we replace the counterfactual $\vs^b_{l_1}$ which has the lowest cost for $\mathcal{C}_j$ with a new candidate counterfactual $\vs_i$, there are two cases. Either, $\textbf{C}^b_{l_1j} > \textbf{C}_{ij}$ or $\textbf{C}^b_{l_1j} \leq \textbf{C}_{ij}$. In case when the candidate $\vs_i$ has lower cost for $\mathcal{C}_j$ than $\textbf{C}^b_{l_1j}$, i.e. $\textbf{C}^b_{l_1j} > \textbf{C}_{ij}$, then the replacement reduces the cost by $\textbf{C}^b_{l_1j} - \textbf{C}_{ij}$. In case when the candidate cost for $\mathcal{C}_j$, $\textbf{C}_{ij}$, is higher than the lowest cost in the best set $\textbf{C}^b_{l_1j}$, i.e. $\textbf{C}^b_{l_1j} \leq \textbf{C}_{ij}$, it means that this replacement will increase the cost for $\mathcal{C}_i$ by $\textbf{C}^b_{l_1j} - min(\textbf{C}_{ij}, \textbf{C}^b_{l_2j})$. Here, $\textbf{C}^b_{l_2j}$ is the second lowest cost counterfactual for $\mathcal{C}_i$. Note that the change in this case will be negative and also depend on the second best counterfactual because once the $\vs^b_{l_1}$ is removed from the set, the best cost for $\mathcal{C}_i$ will either be for $\vs^b_{l_2}$ or $\vs_i$, hence we take the minimum of those two and then take the difference as the increase in cost. Please refer to Algorithm \ref{alg:theorem} for a cognitively easier way to understand the proof.
\end{proof}

\subsubsection{Other Methods}
\label{sec:supp_other_methods}
In this section, we describe some of the optimization methods used by relevant baselines.

\textbf{1. DICE} \citep{dice_fat} perform gradient-based optimization in this continuous space while optimizing for objective defined in Section \ref{sec:supp_other_objective}. Their final objective function is defined as

\[ C(\vx) = \argmin\limits_{\vc_1,\hdots, \vc_k} \frac{1}{k}\sum_{i=1}^{k} \text{loss}(f(\vc_i), y) + \frac{\lambda_1}{k} \sum_{i=1}^{k} \text{dist}(\vc_i,\vx) - \lambda_2 ~ \text{dpp\_diversity}(\vc_1, \hdots, \vc_k)\]
where $\vc_i$ is a counterfactual, $k$ is the number of counterfactuals, f(.) is the black box ML model, yloss(.) is the metric which minimizes the distance between models prediction and the desired outcome $y$. dpp\_diversity(.) is the diversity metric as defined in Section \ref{sec:supp_other_objective} and $\lambda_1$ and $\lambda_2$ are hyperparameters to balance the components in the objective. Please refer to \cite{dice_fat} for more details. 

\textbf{2. FACE} \citep{face} operates under the idea that to obtain actionable counterfactuals they need to be connected to the user state via paths that are probable under the original data distribution aka high-density path. They construct two different types of graphs based on nearest neighbors (Face-knn) and the $\epsilon$-graph (Face-Eps). They define geodesic distance which trades-off between the path length and the density along this path. Lastly, they use the Shortest Path First Algorithm (Dijkstra’s algorithm) to get the final counterfactuals. Please refer to \citep{face} for more details.  

\textbf{3. Actionable Recourse} \citep{ar_fat} tries to find an action set $\va$ for a user such that taking the action changes the black-box models decision to the desired outcome class, denoted by $+1$. They try to minimize the cost incurred by the user while restricting the set of actions within an action set $A(x)$. The set $A(x)$ imposes constraints related to feasibility and actionability with respect to features. They optimize the log-percentile shift objective (see Section \ref{sec:supp_other_objective}). Their final optimization equation is 
\[\min cost(\va; \vx) ~~s.t.~~ f(\vx +\va) = +1,~ \va \in A(\vx)\] which is cast as an Integer Linear Program \citep{mittlemanmip2018_ilp} to provide users with recourses. Their publicly available implementation is limited to a binary case for categorical features,\footnote{Please refer to the this example where they mention about these restricted abilities \href{https://github.com/ustunb/actionable-recourse/blob/master/examples/ex_01_quickstart.ipynb}{https://github.com/ustunb/actionable-recourse/blob/master/examples/ex\_01\_quickstart.ipynb}} hence we demonstrate results on the binarized version of the dataset.

\end{document}